\documentclass[letterpaper]{article} 
\usepackage{aaai2026}  
\usepackage{times}  
\usepackage{helvet}  
\usepackage{courier}  
\usepackage[hyphens]{url}  
\usepackage{graphicx} 
\usepackage{natbib}  
\usepackage{caption} 
\usepackage{algorithm}
\usepackage[noend]{algorithmic}
\usepackage{xcolor}
\usepackage{soul}
\usepackage{url}

\usepackage{lipsum} 
\usepackage{newfloat}
\usepackage{listings}
\DeclareCaptionStyle{ruled}{labelfont=normalfont,labelsep=colon,strut=off} 
\floatstyle{ruled}
\newfloat{listing}{tb}{lst}{}
\floatname{listing}{Listing}
\title{{\algname}: Test-Time Adaptation with Negative Data Augmentation}
\author{
    Ruxi Deng\textsuperscript{\rm 1},  Wenxuan Bao\textsuperscript{\rm 1},  Tianxin Wei\textsuperscript{\rm 1},  Jingrui He\textsuperscript{\rm 1} \\
}
\affiliations{
    \textsuperscript{\rm 1}University of Illinois Urbana-Champaign\\

    \{ruxid2, wbao4, twei10, jingrui\}@illinois.edu
}

\usepackage{algorithm}

\usepackage{graphicx}  
\usepackage{bbding}
\usepackage{xcolor}
\usepackage{multirow}
\usepackage{booktabs}
\usepackage{makecell}

\newcommand{\meansd}[2]{#1\;\textnormal{\scriptsize({#2})}}

\usepackage{colortbl}
\definecolor{Gray}{gray}{0.90}
\newcolumntype{g}{>{\columncolor{Gray}}c}

\usepackage{amsmath,amsfonts,bm}

\def\eqref#1{equation~\ref{#1}}

\def\1{\bm{1}}

\def\vzero{{\bm{0}}}

\def\vd{{\bm{d}}}

\def\vn{{\bm{n}}}

\def\vp{{\bm{p}}}
\def\vq{{\bm{q}}}

\def\vt{{\bm{t}}}

\def\vv{{\bm{v}}}
\def\vw{{\bm{w}}}
\def\vx{{\bm{x}}}

\def\mI{{\bm{I}}}

\def\mR{{\bm{R}}}

\DeclareMathAlphabet{\mathsfit}{\encodingdefault}{\sfdefault}{m}{sl}
\SetMathAlphabet{\mathsfit}{bold}{\encodingdefault}{\sfdefault}{bx}{n}

\def\gE{{\mathcal{E}}}

\def\gN{{\mathcal{N}}}

\newcommand{\R}{\mathbb{R}}

\newcommand{\Cov}{\mathrm{Cov}}

\DeclareMathOperator*{\argmax}{arg\,max}

\DeclareMathOperator{\sign}{sign}

\usepackage{amsthm}
\usepackage{thm-restate}

\theoremstyle{plain}

\newtheorem{theorem}{Theorem}[section]

\newtheorem{corollary}[theorem]{Corollary}

\theoremstyle{definition}

\theoremstyle{remark}
\newtheorem{remark}[theorem]{Remark}

\newcommand{\algname}{\texttt{Panda}} 

\newcommand{\vcls}{v_{\text{cls}}}
\newcommand{\vcorr}{v_{\text{corr}}}

\newcommand{\vvcls}{\vv_{\text{cls}}}
\newcommand{\vvcorr}{\vv_{\text{corr}}}

\newcommand{\normalize}{\mathrm{normalize}}

\usepackage{tikz}
\usetikzlibrary{tikzmark,calc}

\begin{document}

\maketitle

\begin{abstract}

Pretrained vision-language models exhibit strong zero-shot classification capabilities, but their predictions degrade significantly under common image corruptions. To improve robustness, many test-time adaptation (TTA) methods adopt positive data augmentation (PDA), which generates multiple views of each test sample to reduce prediction variance. However, these methods suffer from two key limitations. First, it introduces considerable computational overhead due to the large number of augmentations required per image. Second, it fails to mitigate prediction bias, where the model tends to predict certain classes disproportionately under corruption, as PDA operates on corrupted inputs and typically does not remove the corruption itself. 
To address these challenges, we propose {\algname}, a novel TTA method based on negative data augmentation (NDA). Unlike positive augmentations that preserve object semantics, {\algname} generates negative augmentations by disrupting semantic content. It divides images into patches and randomly assembles them from a shared patch pool. These negatively augmented images retain corruption-specific features while discarding object-relevant signals. We then subtract the mean feature of these negative samples from the original image feature, effectively suppressing corruption-related components while preserving class-relevant information. This mitigates prediction bias under distribution shifts. Importantly, {\algname} allows augmentation to be shared across samples within a batch, resulting in minimal computational overhead. {\algname} can be seamlessly integrated into existing test-time adaptation frameworks and substantially improve their robustness. Our experiments indicate that {\algname} delivers superior performance compared to PDA methods, and a wide range of TTA methods exhibit significantly enhanced performance when integrated with {\algname}. Our code is available at \url{https://github.com/ruxideng/Panda}.

\end{abstract}

\section{Introduction}

Pretrained vision-language models (VLMs) such as CLIP \cite{clip} have demonstrated strong zero-shot generalization across a wide range of vision tasks \cite{regionclip,denseclip,styleclip}. However, their performance can degrade significantly under image corruptions \cite{corruption}. One major reason is that distribution shifts caused by corruptions introduce prediction bias: the model tends to misclassify corrupted images into specific categories \cite{sar,deyo}. This bias arises because corruption patterns may be inadvertently used as spurious features, leading the model to associate them with particular classes regardless of the underlying object content \cite{corruption,deyo}.

\begin{figure}
    \centering
    \includegraphics[width=\linewidth]{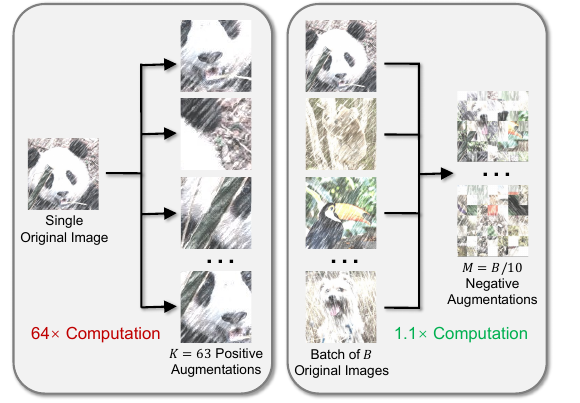}
    \caption{
        Comparison of positive (left) and negative (right) data augmentation.
        Left: PDA used in previous TTA algorithms generates $K$ class-preserving views per image, resulting in high computational cost.
        Right: NDA used in {\algname} generates $M$ class-agnostic corrupted views shared across a batch of $B$ images, incurring minimal overhead.
    }
    \label{fig:pda_vs_nda}
\end{figure}

To mitigate the impact of image corruptions, recent work has explored test-time adaptation (TTA), which adapts the model on-the-fly using only unlabeled test data. Among various strategies, one popular line of methods is based on positive data augmentation (PDA) \cite{augmix}. These methods generate multiple random augmented views for each test image, aggregate the logits to improve prediction robustness \cite{vte,zero} or update the model \cite{memo,tpt,tps}. However, PDA-based TTA suffers from several limitations. In terms of computational efficiency, it requires generating $K$ augmented views independently for each test sample (typically $K=63$), leading to significant increases in test-time cost. In terms of effectiveness, while PDA preserves semantic content, it also tends to retain corruption patterns. Even after averaging multiple augmented views, these corruptions typically persist and can often amplify prediction bias (see our results in Figure \ref{fig:prediction_bias}). Consequently, existing PDA-based TTA methods are generally ineffective in mitigating prediction bias. 

Motivated by these limitations, we propose {\algname}: test-time adaPtAtion with Negative Data Augmentation. Unlike PDA, which preserves object semantics, our negative data augmentation (NDA) intentionally \textit{disrupts semantic content while retaining corruption characteristics}. An illustration is provided in Figure \ref{fig:pda_vs_nda}. Specifically, {\algname} divides all images within a test batch into small patches to create a shared patch pool, then randomly samples patches from this pool to assemble new, recombined negative augmentations. These negatively augmented images obscure object semantics but preserve the underlying corruption patterns. By subtracting the features of these negatively augmented images from the original image embeddings, {\algname} effectively suppresses corruption-related signals while retaining class-relevant information, thus mitigating prediction bias (see Figure \ref{fig:prediction_bias}).

Furthermore, since negative augmentations are shared across samples within the same batch, {\algname} avoids the computational overhead of generating augmentations for each individual image, making it more efficient than PDA-based TTA methods. Additionally, as {\algname} modifies only the forward propagation step, it can be seamlessly integrated with existing TTA frameworks. Our evaluation on standard corruption benchmarks shows that {\algname} consistently enhances performance across various TTA algorithms with minimal computational cost. We summarize our contributions as follows:
\begin{itemize}
    \item We identify that existing TTA methods based on positive data augmentation often incur high inference costs while failing to effectively reduce prediction bias. 
    \item We propose {\algname}, a novel TTA method that leverages negative data augmentation to suppress corruption-related features in image embeddings and significantly reduce prediction bias.
    \item Extensive experiments show that {\algname} achieves better performance than existing PDA-based methods with substantially lower computational cost. In addition, {\algname} can be integrated with most existing TTA algorithms to significantly enhance their robustness.
\end{itemize}

\section{Related Works}

\textbf{Data augmentation} refers to the process of generating additional samples by applying transformations to existing data, and is widely used during model training to improve generalization. Broadly, data augmentation methods can be categorized into positive and negative augmentation. Positive data augmentation (PDA) operates on a single image and aims to preserve its semantic content. Common PDA techniques such as Cutout \cite{cutout} and AugMix \cite{augmix} act as regularizers that encourage robustness to perturbations. In contrast, negative data augmentation (NDA) typically involves combining information from multiple images and often alters their semantics. Examples include MixUp \cite{mixup} and CutMix \cite{cutmix}, which have been shown to further enhance generalization by promoting smoother decision boundaries. It is important to note that these methods are primarily designed for use during model training, rather than at test time.

\textbf{Test-time adaptation} (TTA) adapts a pre-trained model to an unlabeled target domain without accessing source data~\cite{tta_survey,tta_survey_2,matcha}. A prominent line of work, exemplified by Tent~\cite{tent}, performs entropy minimization by updating the model’s normalization layers. Follow-up methods such as ETA/EATA~\cite{eata}, SAR~\cite{sar}, and DeYO~\cite{deyo} improve adaptation stability by incorporating entropy-aware sample selection or weighting strategies. Notably, DeYO also leverages negative data augmentation to guide this process, using the prediction difference between the original image and its negatively augmented counterpart to assess reliability. Training-free TTA methods~\cite{dmn,tda,latte} avoids any model updates and instead modifies predictions directly based on inter-sample similarity. Several TTA approaches incorporate positive data augmentation (PDA) to enhance test-time robustness, typically using AugMix to generate multiple views per image. Methods such as VTE and Zero aggregate predictions across these views to reduce randomness, while MEMO~\cite{memo}, TPT~\cite{tpt}, and TPS~\cite{tps} minimize marginal entropy to enforce consistency across augmented inputs.



\section{Challenges}

\paragraph{Preliminary} CLIP \cite{clip} is a vision-language model (VLM) with an image encoder $\gE_v$ and a text encoder $\gE_t$, which aligns images with their corresponding textual descriptions. By pretraining on a large-scale image-text dataset, CLIP is capable of zero-shot prediction. Specifically, for a classification task with $C$ classes, the text encoder $\gE_t$ embeds class descriptions (e.g., ``\texttt{a photo of a \{class\}}'') into normalized text embeddings $\vt_1, \ldots, \vt_C \in \R^D$. Given a test image, the image encoder $\gE_v$ produces a normalized image feature $\vv_i \in \R^D$, and the prediction is made by assigning the image to the class with the highest similarity score, i.e., $\hat{y}_i = \arg\max_{c}~ \vv_i^\top \vt_c$. 

\begin{figure}
    \centering
    \includegraphics[width=0.9\linewidth]{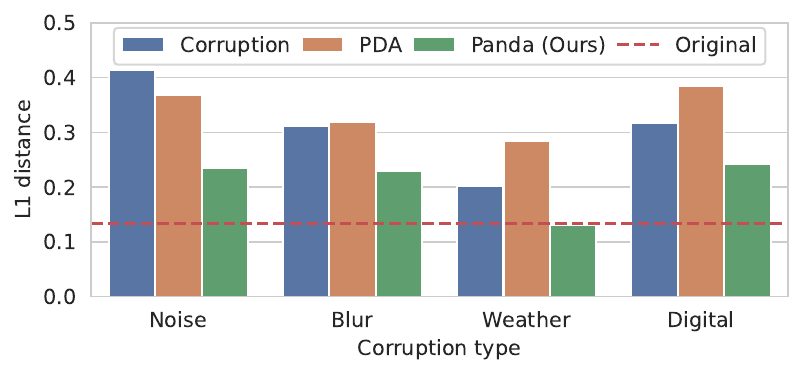}
    \caption{
        Distribution distance between ground-truth and soft prediction distributions under four corruption categories. 
        \textit{Original} denotes the uncorrupted CIFAR-10 dataset. 
        Larger distribution distance indicates greater prediction bias. 
        Corruptions introduce significant bias that positive data augmentation often fails to mitigate. 
        In contrast, {\algname} effectively reduces this bias. 
        See Figure~\ref{fig-l1-all} in Appendix~\ref{appendix-a-analysis} for results on all 15 corruption types.
        }
    \label{fig:prediction_bias}
\end{figure}

\paragraph{Prediction bias from corruptions}
Although CLIP exhibits strong zero-shot generalization capabilities, its classification accuracy often degrades in the presence of common corruptions \cite{corruption}. Corruptions such as Gaussian noise and defocus blur can be encoded into the image embeddings, introducing bias into the representation. When these biases correlate spuriously with text embeddings, they lead to \textit{prediction bias}, where corrupted images are disproportionately assigned to certain classes. To quantify this effect, we evaluate the distribution distance between the ground-truth label distribution and the soft prediction distribution on CIFAR-10-C \cite{cifar,corruption}. A larger distance indicates more severe prediction bias. As shown in Figure \ref{fig:prediction_bias}, each type of corruption significantly increases prediction bias. Prediction bias not only directly reduces classification accuracy but also poses a critical challenge for entropy-based TTA methods. Since these methods treat soft predictions as pseudo-labels, the presence of prediction bias can amplify itself during adaptation and potentially lead to model collapse \cite{sar,tta_pitfall}. Therefore, reducing prediction bias is essential for robust TTA.

\paragraph{PDA fails to alleviate prediction bias}
Many TTA methods \cite{memo,tpt,tps,vte,zero} use positive data augmentation (PDA) \cite{augmix} to improve robustness to image corruptions. These methods generate multiple semantic-preserving views for each test image and average the predictions from high-confidence views. However, as shown in Figure \ref{fig:prediction_bias}, we observe that PDA often fails to alleviate prediction bias and sometimes even amplifies it. This occurs because positive augmentations do not remove the corruption present in the image. In addition, selecting and averaging high-confidence views can unintentionally strengthen the influence of corruption within the image embedding. In comparison, our proposed method {\algname} adopts negative data augmentation (NDA) to reduce prediction bias across all corruption types. In the following section, we describe how we design the NDA process and how its outputs are used to debias image embeddings and adapt the model accordingly.

\section{Algorithm}

\begin{figure*}
    \centering
    \includegraphics[width=1\linewidth]{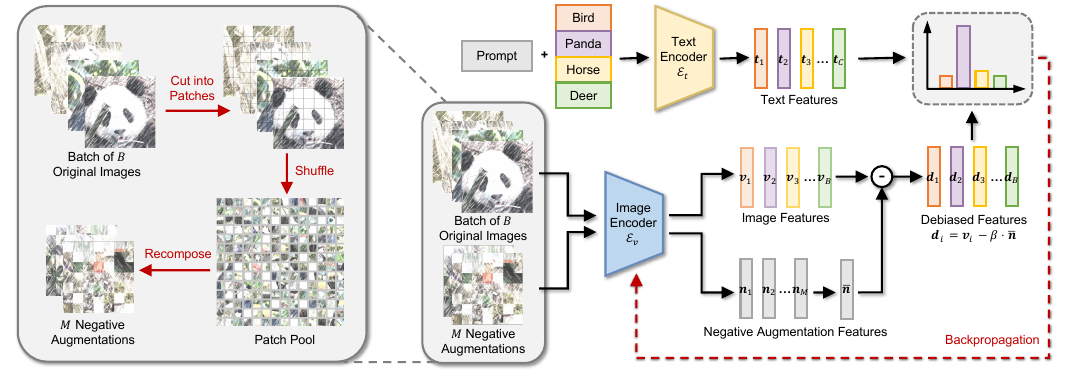}
    \caption{Overview of {\algname}. Given a batch of $B$ original images, $M$ negatively augmented images are generated by cutting the originals into patches, shuffling, and recomposing. Both original and negative augmented images are encoded by the image encoder. The average of the negative embeddings serves as a corruption prototype and is subtracted from the original embeddings to suppress corruption-related features. Final predictions are obtained by comparing the debiased features with text embeddings.}
    \label{fig-panda}
\end{figure*}

Figure~\ref{fig-panda} gives an overview of our proposed algorithm {\algname}, which consists of three main steps: 
\begin{enumerate}
    \item For each test batch of $B$ images, we generate $M \ll B$ negative augmentations. (Subsection \ref{subsec:algo:nda})
    \item We encode both the test samples and the negative augmentations into image embeddings. The average embedding of the negative augmentations is used to offset each test sample, reducing the influence of corruption. (Subsection \ref{subsec:algo:offset}). 
    \item The debiased features are then passed to any existing TTA method to refine predictions or adapt the model. (Subsection \ref{subsec:algo:adapt})
\end{enumerate}


\subsection{Negative Data Augmentation} \label{subsec:algo:nda}

We begin by introducing our negative data augmentation (NDA) strategy. The goal of NDA is to generate images from a batch of test samples that preserve corruption-related information while disrupting object-related content. Specifically, our NDA takes a batch of $B$ images $\{\vx_i\}_{i=1}^B$ as input. Each image of size $H \times W$ is first partitioned into $\frac{H}{H_p} \times \frac{W}{W_p}$ non-overlapping patches of size $H_p \times W_p$. All patches from the batch are collected and randomly shuffled to form a patch pool. From this pool, patches are selected and combined to recompose $M$ \textit{negatively augmented images} $\{\vx^{-}_j\}_{j=1}^M$, where each negative augmentation image is constructed by combining patches such that every patch is used at most once within the batch, and the size of each image remains $H \times W$. These negatively augmented images are fed into the image encoder $\gE_v$ alongside the original batch. 

Compared to PDA, our NDA strategy requires significantly fewer augmentations. PDA typically generates $K$ augmentations for each image in the batch (commonly $K = 63$), resulting in a forward pass cost of $K + 1$ times. In contrast, NDA does not need to preserve object information for individual images, allowing all samples in a batch to share a common set of $M$ negative augmentations. In practice, we typically set $M = B / 10$, which is much smaller than the batch size $B$, leading to substantial computational savings.


\subsection{Offset} \label{subsec:algo:offset}
We feed both the original $B$ test images and the $M$ negatively augmented images into the image encoder, obtaining image embeddings $\{\vv_i = \gE_v(\vx_i)\}_{i=1}^B$ for the original inputs and $\{\vn_j = \gE_v(\vx^-_j)\}_{j=1}^M$ for the negatively augmented images. We aggregate the negatively augmented embedding by computing their average: 
\begin{align}
    \bar{\vn} = \frac{1}{M} \sum_{j=1}^{M} \vn_j, 
\end{align}
Due to the patch-level shuffling and averaging across multiple negatively augmented images, the resulting embedding $\bar\vn$ contains minimal object-related information while retaining corruption-related characteristics. We use $\bar\vn$ to offset the original image embeddings in order to suppress the corruption components: 
\begin{align}
    \vd_i = \vv_i - \beta \cdot \bar\vn, \quad i = 1, \cdots, B,
\end{align}
where $\beta > 0$ is a hyperparameter controlling the offset ratio. We use the debiased embedding $\{\vd_i\}_{i=1}^B$ as a replacement of the original image embedding $\{\vv_i\}_{i=1}^B$. When not combined with other TTA algorithms, the prediction is given by $\hat{y}_i = \argmax_c \vd_i^\top \vt_c$, where $\vt_c$ is the text embedding for class $c$. 


In Theorem \ref{thm:offset} below, we use an one-dimensional example to justify the validity of this offset approach. Note that this conclusion can be generalized to high-dimensional settings. The proof and its high-dimensional generalization are provided in Appendix \ref{appendix-c-proof}.

\begin{restatable}[Offsetting leads to accuracy gain]{theorem}{thmoffset}
\label{thm:offset}
Consider a binary classification problem where the input feature $v$ can be decomposed into two independent components: $v = \vcls + \vcorr$, where $\vcls \sim \gN(0, 1)$ denotes the class-relevant component and $\vcorr \sim \gN(0, s^2)$ denotes the corruption-related component, with $s > 0$ representing corruption severity. Let the ground-truth label be $y = \sign(\vcls)$, and the classifier be $\hat{y} = \sign(v)$. Then the classification accuracy is
\begin{align}
    \Pr(\sign(v) = y) = \frac{1}{2} + \frac{1}{\pi} \cdot \arctan\left(\frac{1}{s}\right).
\end{align}
Now consider a negatively augmented feature $n \sim \gN(0, s^2)$ such that the correlation $\rho(n, \vcls) = 0$ and $\rho(n, \vcorr) = r > 0$. Then offsetting $v$ using $n$ yields improved accuracy:
\begin{align}\begin{split}
    &\Pr(\sign(v - \beta \cdot n) = y)  \\
    &\quad = \frac{1}{2} + \frac{1}{\pi} \cdot \arctan\left(\frac{1}{s \cdot \sqrt{1 - r^2 + (\beta - r)^2}}\right), 
\end{split}\end{align}
which is maximized when $\beta = r$. 
\end{restatable}



Theorem~\ref{thm:offset} shows that when the negative augmentation effectively removes object information and retains part of the corruption characteristics, our offset strategy can suppress the corruption component in the original feature down to a minimum of $\sqrt{1 - r^2}$ times its original magnitude, leading to improved accuracy. This result highlights the effectiveness of using negative augmentation to isolate and suppress corruption in the feature space. 


\renewcommand{\algorithmicrequire}{\textbf{Input:}}
\renewcommand{\algorithmicensure}{\textbf{Output:}} 

\begin{algorithm}[H]
    \caption{Tent \begin{tikzpicture}[remember picture, overlay]
            \draw[line width=0pt, draw=blue!30, rounded corners=2pt, fill=blue!30, fill opacity=0.2]
                ($(pic cs:a) + (-1pt,8pt)$) rectangle ($(pic cs:b)+(40pt,-3pt)$);
            \end{tikzpicture}+ {\algname}} \label{algo-tentourss}
    \small
    \begin{algorithmic}[1]
        \REQUIRE Test data stream $\{\mathcal{X}_t\}_{t=1}^T$ with $\mathcal{X}_t = \{\vx_i\}_{i=1}^B$; image encoder $\gE_v(\cdot; \vw)$ with parameters $\vw$ (e.g., LayerNorm); text embeddings $\{\vt_c\}_{c=1}^C$; learning rate $\eta$; patch size $H_p \times W_p$; offset ratio $\beta$; number of negative augmentations $M = \lceil B/10 \rceil$
        \ENSURE Predictions $\{\hat{y}_i\}_{i=1}^B$ for each batch

        \FOR{each test batch $\mathcal{X}_t$}
            \STATE \begin{tikzpicture}[remember picture, overlay]
            \draw[line width=0pt, draw=blue!30, rounded corners=2pt, fill=blue!30, fill opacity=0.2]
                ($(pic cs:a) + (-5pt,8pt)$) rectangle ($(pic cs:b)+(214pt,-23pt)$);
            \end{tikzpicture} \textcolor{gray}{\textit{\# Negative data augmentation}}
            \STATE Generate $M$ negative augmentations $\{\vx_j^-\}_{j=1}^M$ by cuting $\vx_i \in \mathcal{X}_t$ into patches, shuffling, and recomposing
            \STATE \textcolor{gray}{\textit{\# Feature encoding}}
            \STATE $[\vv_1, \cdots, \vv_B] \leftarrow \normalize\left(\gE_v([\vx_1, \cdots, \vx_B]; \vw)\right)$
            \STATE \begin{tikzpicture}[remember picture, overlay]
            \draw[line width=0pt, draw=blue!30, rounded corners=2pt, fill=blue!30, fill opacity=0.2]
                ($(pic cs:a) + (-5pt,8pt)$) rectangle ($(pic cs:b)+(214pt,-35pt)$);
            \end{tikzpicture}$[\vn_1, \cdots, \vn_M] \leftarrow \normalize\left(\gE_v([\vx_1^-, \cdots, \vx_M^-]; \vw)\right)$
            \STATE \textcolor{gray}{\textit{\# Feature debiasing}}
            \STATE $\bar{\vn} \leftarrow \frac{1}{M} \sum_{j=1}^{M} \vn_j$
            \STATE $\vd_i \leftarrow \vv_i - \beta \cdot \bar{\vn}, \text{ for } i = 1, \cdots, B$
            \STATE \textcolor{gray}{\textit{\# Adaptation}}
            \STATE $\text{logits}_i \leftarrow 100 \cdot \vd_i^\top [\vt_1, \cdots, \vt_C], \text{ for } i = 1, \cdots, B$
            \STATE $\mathcal{L} \leftarrow \frac{1}{B} \sum_{i=1}^{B} \mathcal{H}(\text{logits}_i)$
            \STATE $\vw \leftarrow \vw - \eta \cdot \nabla_{\vw} \mathcal{L}$ \hfill 
            \textcolor{gray}{\textit{\# update parameters of $\gE_v$}}
            \STATE $\hat{y}_i \leftarrow \argmax_c (\vd_i^\top \vt_c)$ for $i = 1, \cdots, B$
        \ENDFOR
        \STATE \textbf{Output} $\{\hat{y}_i\}_{i=1}^B$
    \end{algorithmic}
\end{algorithm}

\renewcommand{\meansd}[2]{#1\scriptsize\ ({#2})}

\begin{table*}
    \centering
    \small
    \setlength{\tabcolsep}{1.4mm}{
        \begin{tabular}{lccccccccccc}
            \toprule
             & & CLIP & Tent & ETA & SAR & DeYO & TPT & DMN-ZS & Zero & TPS & BAT \\
            \midrule
            \multirow{3}{*}{\makecell[c]{ViT-B/32 on \\CIFAR-10-C}}
            & Baseline 
            & \meansd{59.0}{0.0} & \meansd{62.8}{0.1} & \meansd{64.9}{0.1} & \meansd{63.3}{0.3} & \meansd{65.5}{0.0} 
            & \meansd{62.2}{0.1} & \meansd{61.6}{0.0} & \meansd{63.2}{0.0} & \meansd{63.7}{0.0} & \meansd{65.7}{0.0} \\
            & +{\algname} 
            & \meansd{61.6}{0.0} & \meansd{71.1}{0.1} & \meansd{68.3}{0.1} & \meansd{70.7}{0.2} & \meansd{67.2}{0.1} 
            & \meansd{63.5}{0.0} & \meansd{63.1}{0.1} & \meansd{64.9}{0.0} & \meansd{65.6}{0.0} & \meansd{68.5}{0.0} \\
            & $\Delta$ 
            & +2.6 & +8.3 & +3.4 & +7.4 & +1.7 & +1.3 & +1.5 & +1.7 & +1.9 & +2.8 \\
            \midrule
            \multirow{3}{*}{\makecell[c]{ViT-B/32 on \\CIFAR-100-C}}
            & Baseline 
            & \meansd{31.8}{0.0} & \meansd{35.7}{0.1} & \meansd{40.8}{0.4} & \meansd{39.1}{0.3} & \meansd{38.0}{0.3} 
            & \meansd{32.1}{0.0} & \meansd{32.2}{0.1} & \meansd{29.5}{0.1} & \meansd{33.1}{0.1} & \meansd{36.5}{0.0} \\
            & +{\algname} 
            & \meansd{33.4}{0.0} & \meansd{38.4}{0.1} & \meansd{43.3}{0.1} & \meansd{41.7}{0.2} & \meansd{42.1}{0.2} 
            & \meansd{34.3}{0.0} & \meansd{34.2}{0.1} & \meansd{33.3}{0.0} & \meansd{35.1}{0.0} & \meansd{37.7}{0.0} \\
            & $\Delta$ 
            & +1.6 & +2.7 & +2.5 & +2.6 & +4.1 & +2.2 & +2.0 & +3.8 & +2.0 & +1.2 \\
            \midrule
            \multirow{3}{*}{\makecell[c]{ViT-B/16 on \\ImageNet-C}}
            & Baseline 
            & \meansd{24.5}{0.0} & \meansd{25.3}{0.0} & \meansd{26.5}{0.0} & \meansd{31.8}{0.2} & \meansd{29.0}{0.2} 
            & \meansd{25.2}{0.0} & \meansd{24.6}{0.0} & \meansd{24.6}{0.1} & \meansd{25.1}{0.1} & \meansd{25.6}{0.0} \\
            & +{\algname} 
            & \meansd{26.2}{0.0} & \meansd{28.2}{0.1} & \meansd{27.9}{0.0} & \meansd{32.4}{0.2} & \meansd{31.2}{0.2} 
            & \meansd{27.1}{0.0} & \meansd{26.2}{0.0} & \meansd{27.1}{0.0} & \meansd{27.2}{0.0} & \meansd{27.2}{0.1} \\
            & $\Delta$ 
            & +1.7 & +2.9 & +1.4 & +0.6 & +2.2 & +1.9 & +1.6 & +2.5 & +2.1 & +1.6 \\
            \bottomrule
        \end{tabular}
    }
    \caption{Comparison of accuracy (mean (s.d.) \%) between the single baseline method and the baseline integrated with {\algname}. Only the average accuracy over 15 corruption types is reported; full per-corruption results are deferred to Appendix~\ref{appendix-d-exp}.}
    \label{table-main}
\end{table*}

\subsection{Combination with TTA Methods} \label{subsec:algo:adapt}
{\algname} modifies only the forward pass with minimal computational overhead, which makes it compatible with most existing TTA methods. For instance, Algorithm 1 gives an example when {\algname} is integrated with Tent \cite{tent}. The standard objective minimizes the entropy of predictions based on the original image features $\vv_i$. In our method, we instead use the debiased features $\vd_i$ to replace $\vv_i$ in the forward computation. This substitution reduces bias in the logits, thereby improving both prediction quality and adaptation stability. As shown in our experiments, existing TTA algorithms can achieve better performance when combined with {\algname}.

\section{Experiments}

In this section, we conduct experiments to investigate the following research questions:
\begin{itemize}
    \item \textbf{RQ1}: Can {\algname} enhance the performance of existing TTA methods significantly in diverse datasets?
    \item \textbf{RQ2}: Compared to PDA methods, does {\algname} achieve better performance with higher efficiency? 
    \item \textbf{RQ3}: Does {\algname} effectively alleviate prediction bias? 
\end{itemize}


\paragraph{Setup}
We conduct experiments on three widely used corruption benchmarks: CIFAR-10-C, CIFAR-100-C, and ImageNet-C, each containing 15 types of common corruptions. Following standard practice \cite{bat,mint}, we evaluate our method at corruption severity level 5. For the backbone, we use ViT-B/32 for CIFAR-10-C and CIFAR-100-C, and ViT-B/16 for ImageNet-C. We adopt a default batch size of 100 for adaptation and use a fixed prompt template (``\texttt{a photo of a \{class\}}'') for the text encoder. 
All images from the corruption datasets are resized to $224 \times 224$ to match the input resolution required by the ViT backbone. We use a default patch size of $H_p = W_p = 32$, meaning each original image is partitioned into $7 \times 7$ non-overlapping patches (i.e., $224 / 32 = 7$ per dimension) during negative augmentation.

\paragraph{Baselines}
Besides CLIP \cite{clip}, we assess the effectiveness of {\algname} in conjunction with nine different TTA baselines, covering diverse adaptation types and models designs. For general TTA methods, we consider Tent \cite{tent}, ETA \cite{eata}, SAR \cite{sar}, and DeYO \cite{deyo}. For CLIP-specific TTA methods, we compare both model-adaptive approaches including TPT \cite{tpt}, TPS \cite{tps}, and BAT \cite{bat}, and training-free methods including DMN-ZS \cite{dmn} and Zero \cite{zero}. Among them, TPT, TPS, and Zero adopt AugMix \cite{augmix} as their positive data augmentation strategy. 
For each baseline, in addition to evaluating the method itself, we also report a combined version that integrates {\algname} by replacing the original image features $\vv_i$ with the debiased features $\vd_i = \vv_i - \beta \cdot \bar{\vn}$ during both logit computation and adaptation.
Unless otherwise specified, all hyperparameters and experimental protocols are optimized using the original baseline alone and kept identical for its {\algname}-augmented counterpart.
Please refer to Appendix~\ref{appendix-d-exp} for full implementation details and hyperparameter values. 

\renewcommand{\meansd}[2]{#1\scriptsize\ ({#2})}

\begin{table}
    \centering
    \label{tab:acc:main_transposed_full}
    \small
    \setlength{\tabcolsep}{2mm}{
        \begin{tabular}{lccc}
            \toprule
            Method & CIFAR-10-C & CIFAR-100-C & ImageNet-C \\
            \midrule
            TPT 
            & \meansd{62.2}{0.1} 
            & \meansd{32.1}{0.0} 
            & \meansd{25.2}{0.0} \\
            Zero 
            & \meansd{63.2}{0.0} 
            & \meansd{29.5}{0.1} 
            & \meansd{24.6}{0.1} \\
            TPS 
            & \meansd{63.7}{0.0} 
            & \meansd{33.1}{0.1} 
            & \meansd{25.1}{0.1} \\
            {\algname} 
            & \meansd{\textbf{71.1}}{0.1} 
            & \meansd{\textbf{38.4}}{0.1} 
            & \meansd{\textbf{28.2}}{0.1} \\
            \bottomrule
        \end{tabular}
    }
    \caption{Comparison of mean accuracy (\%) between {\algname} and PDA methods (TPT, Zero, and TPS). Only the average accuracy over 15 corruption types is reported; full per-corruption results are deferred to Appendix~\ref{appendix-d-exp}.}
    \label{table-pda}
\end{table}

\begin{table*}
\centering
\small
\begin{tabular}{lcccccccccc}
\toprule
Method & CLIP & Tent & ETA & SAR & DeYO & TPT & DMN & Zero & TPS & BAT \\
\midrule
Baseline Time & 17s & 25s & 21s & 31s & 27s & 22min21s & 22s & 8min51s & 9min32s & 28s \\
+{\algname} Time & 18s & 27s & 23s & 34s & 28s & 22min39s & 23s & 8min55s & 9min37s & 30s \\
Overhead & 5.9\% & 8.0\% & 9.5\% & 9.7\% & 3.7\% & 1.3\% & 4.5\% & 0.8\% & 0.9\% & 7.2\% \\
\bottomrule
\end{tabular}
\caption{Comparison of testing time with baselines and baselines+{\algname} for ViT-B/32 on CIFAR-10.}
\label{table-time}
\end{table*}


\paragraph{Combination with existing TTA methods (RQ1)} We evaluate the effectiveness of {\algname} by integrating it into a wide range of existing TTA methods across three corruption benchmarks. As shown in Table~\ref{table-main}, {\algname} consistently improves the performance of all baselines across all datasets, demonstrating its strong compatibility with existing methods and its ability to effectively leverage high-quality negative augmentations to suppress spurious features in original images under various settings. 

On CIFAR-10-C, {\algname} brings substantial gains to general TTA methods such as Tent (+8.3\%), ETA (+3.4\%), and SAR (+7.4\%), with an average improvement of +3.3\% across all baselines. On CIFAR-100-C, {\algname} achieves an average improvement of +2.2\% over all baselines, with the highest gain of +4.1\%. On ImageNet-C, it brings an average improvement of +2.0\%, with a maximum gain of +2.9\%.

Notably, {\algname} can also be integrated with PDA methods and achieve better performance, as evidenced by improvements on TPT (+2.2\%), TPS (+3.8\%), and Zero (+2.0\%) on CIFAR-100-C. In addition, {\algname} can also be integrated with the NDA method DeYO to enhance its ability, which as shown by improvements on DeYO (+4.1\%) on CIFAR-100-C.


\paragraph{Compared with PDA methods (RQ2)} Beyond the distributional distance analysis, we conduct controlled experiments under identical settings to compare {\algname} with positive data augmentation  methods, and observe from Table~\ref{table-pda} that {\algname} significantly outperforms all PDA baselines (TPT~\cite{tpt}, Zero~\cite{zero}, and TPS~\cite{tps}) by notably reducing prediction bias and achieving higher accuracy.

\paragraph{Efficiency (RQ2)} We evaluate the computational overhead introduced by integrating {\algname} by comparing runtime of baseline methods with and without {\algname}. From Table~\ref{table-time}, incorporating {\algname} results in less than a 10\% increase in runtime compared to their standalone counterparts, which demonstrates that {\algname} achieves performance improvements with minimal efficiency overhead. Moreover, in contrast to positive data augmentation methods such as TPS, TPT, and Zero that generate numerous independent augmented views for each test sample, {\algname} with only a small number of negative augmentations consequently achieves substantially lower computational cost and runtime.


\begin{figure}
    \centering
    \includegraphics[width=1\linewidth]{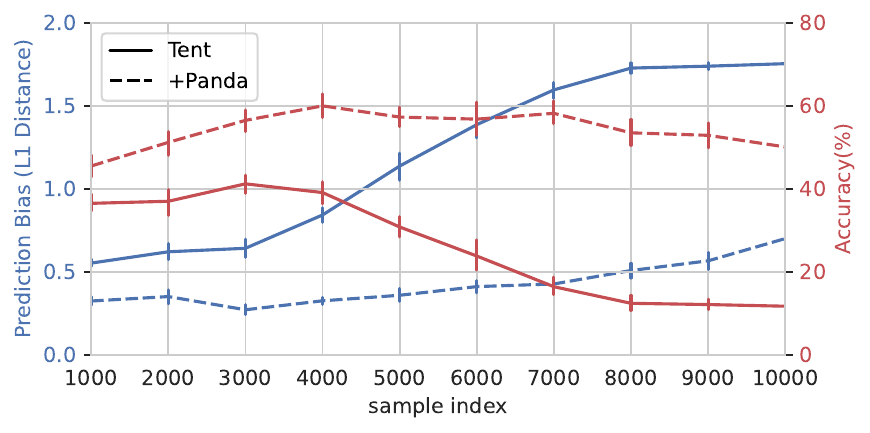}
    \caption{Prediction bias and accuracy (\%) measured across the test stream, divided into 10 consecutive chunks (each of 1,000 samples).}
    \label{fig-bias-evolve}
\end{figure}

\paragraph{Alleviate prediction bias (RQ3)}
In Figure~\ref{fig:prediction_bias} and Figure~\ref{fig-l1-all} in Appendix~\ref{appendix-a-analysis}, we have verified that applying {\algname} alone can effectively reduce the model's prediction bias. Here we further investigate whether {\algname} can alleviate prediction bias when combined with other TTA methods. To this end, we choose Tent \cite{tent} as a baseline, as it is known to suffer from prediction bias \cite{sar}. We conduct experiments on the gaussian noise corruption from CIFAR-10-C. Fixing the data order, we divide the 10,000 test samples into 10 consecutive chunks (1–1000, 1001–2000, …). For each chunk, we compute (1) the prediction bias, measured as the L1 distance between the ground-truth label distribution and the soft prediction distribution, and (2) the classification accuracy. 
As shown in Figure~\ref{fig-bias-evolve}, Tent gradually accumulates prediction bias as more test samples are processed. As a result, its accuracy improves slightly only at the beginning but quickly degrades, eventually leading to model collapse. In contrast, Tent + {\algname} maintains consistently lower prediction bias and achieves substantially higher accuracy throughout the entire test stream. These results highlight that {\algname} not only alleviates prediction bias in a static sense but also mitigates its accumulation throughout the test-time adaptation process, leading to improved performance.


\renewcommand{\meansd}[2]{#1\scriptsize\ ({#2})}

\begin{table}
    \centering
    \label{tab:acc:main_transposed}
    \small
    \setlength{\tabcolsep}{1mm}{
        \begin{tabular}{lccc}
            \toprule
            Method & CIFAR-10-C & CIFAR-100-C & ImageNet-C \\
            \midrule
            CLIP 
            & \meansd{59.0}{0.0} 
            & \meansd{31.8}{0.0} 
            & \meansd{24.5}{0.0} \\
            Select \& weight
            & \meansd{65.5}{0.0} 
            & \meansd{38.0}{0.3} 
            & \meansd{29.0}{0.2} \\
            Offset (Ours)
            & \meansd{\textbf{68.3}}{0.0} 
            & \meansd{\textbf{43.3}}{0.1} 
            & \meansd{\textbf{29.4}}{0.0} \\
            \midrule
        \end{tabular}
    }
    \caption{Comparison of mean accuracy (\%) between  for negative augmentation strategies in DeYO and {\algname}. Selecting and weighting testing samples using NDA is from DeYO and offseting bias features is from {\algname}. Only the average accuracy over 15 corruption types is reported; full per-corruption results are deferred to Appendix~\ref{appendix-d-exp}.}
    \label{table-deyo}
\end{table}

\begin{figure*}[t]
    \centering
    \includegraphics[width=1.0\linewidth]{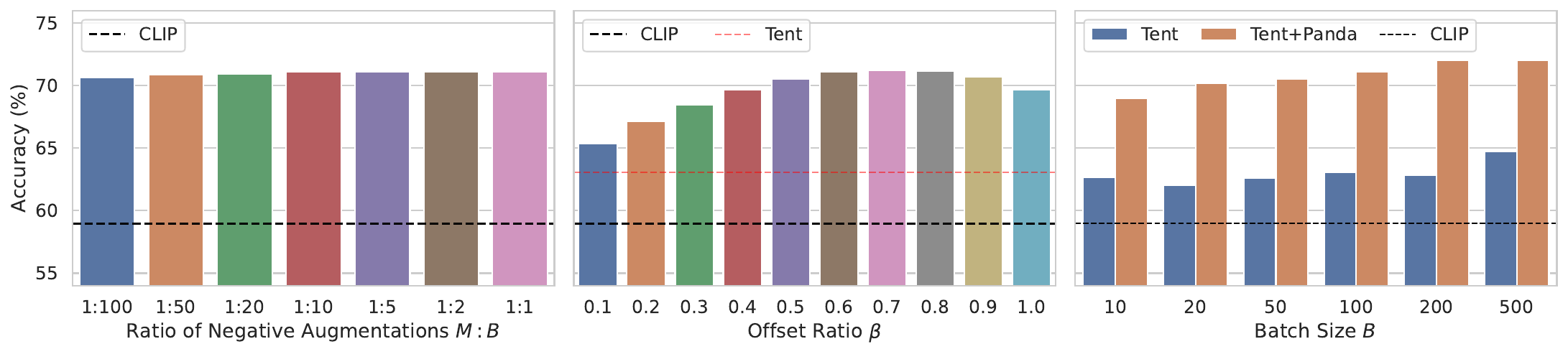}
    \caption{Sensitivity analysis of {\algname}. Left: accuracy under different ratios of $M:B$, where $M$ is the number of negative augmentations per batch and $B$ is the batch size. 
Middle: accuracy across a range of offset ratios $\beta$ used in patch translation. 
Right: accuracy under varying batch sizes, comparing Tent and Tent+{\algname}.}
    \label{fig-mb-offset-batch}
\end{figure*}

\begin{figure*}[t]
    \centering
    \includegraphics[width=\linewidth]{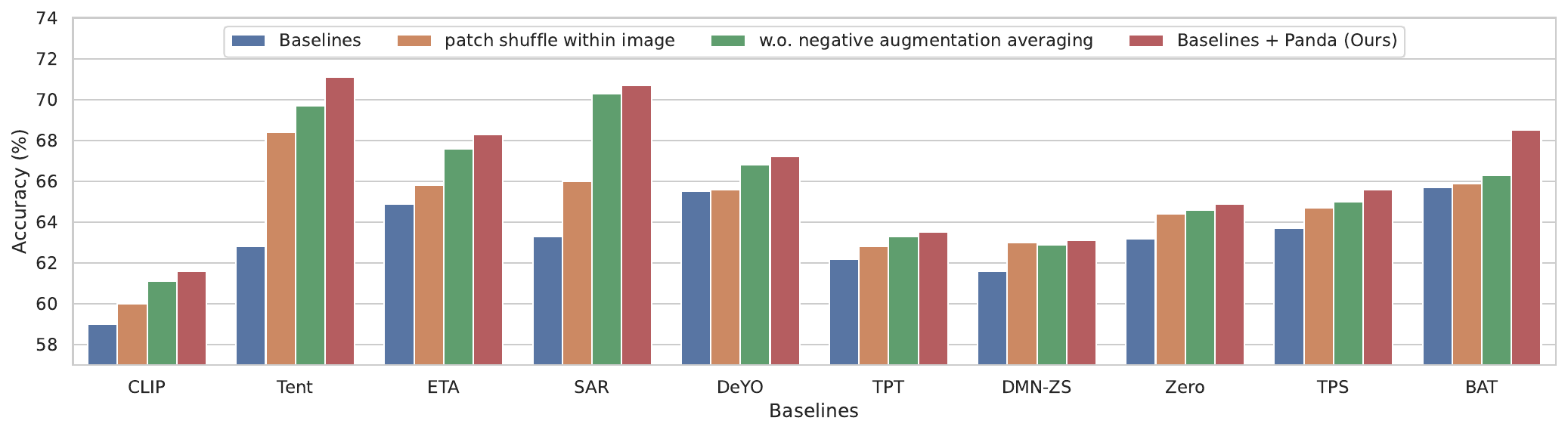}
    \caption{Ablation study comparing the full version of {\algname} with its decomposed variants across different TTA baselines. 
Variants include disabling {\algname}, shuffling patches within individual images, and removing augmentation feature averaging.}
    \label{fig-ablation}
\end{figure*}

\paragraph{Comparison of negative augmentation strategies with other NDA methods}
To evaluate the effectiveness of the negative data augmentation used in {\algname}, we compare it with a strong NDA-based baseline, DeYO~\cite{deyo}. DeYO estimates prediction confidence by measuring the discrepancy between predictions on original and negatively augmented images, and uses this confidence to guide sample selection and weighting during adaptation.
For a fair comparison, we remove both the NDA-relevant components from DeYO, and instead adopt the NDA generation and offset mechanism used in {\algname}. As shown in Table~\ref{table-deyo}, the negative augmentation strategy employed by {\algname} achieves superior performance. This result demonstrates that {\algname} produces higher-quality negative augmentations that more effectively suppress prediction bias on corrupted data, outperforming existing state-of-the-art NDA approaches in test-time adaptation.

\paragraph{Ratio between $B$ and $M$}
$B$ denotes the batch size and $M$ represents the number of negative augmentation images generated per batch. We investigate whether the performance of {\algname} degrades significantly as the $M/B$ ratio decreases. As shown in Figure~\ref{fig-mb-offset-batch} (left), {\algname} maintains stable and strong performance across different $M/B$ settings. The consistently high performance under extremely low $M/B$ ratios indicates that generated negative augmentation images exhibit strong information-sharing capacity, allowing a small number of augmentations to benefit a large number of samples effectively.

\paragraph{Hyperparameter sensitivity}
We conduct sensitivity analyses of {\algname} with respect to the offset ratio $\beta$, batch size, and learning rate. As shown in Figure~\ref{fig-mb-offset-batch} (middle), {\algname} consistently yields significantly better performance across a wide range of offset ratio $\beta$. This shows that {\algname} can consistently enhance the performance of baselines across a wide range of offset ratio $\beta$, without relying on a specific value. Moreover, as illustrated in Figure~\ref{fig-mb-offset-batch} (right), when combined with {\algname}, Tent maintains a consistently large performance gain over the original Tent baseline across different batch sizes. This indicates that {\algname} can effectively enhance the performance of the base method under varying batch size settings. Detailed results for other baselines under different learning rates are provided in Appendix~\ref{appendix-d-exp}.

\paragraph{Ablation study} 
We conduct an ablation study by decomposing components of {\algname}. Specifically, we compare the full version of {\algname} with several variants: (1) disabling {\algname} entirely (2) performing negative augmentation based on the image itself rather than across the whole batch of images, and (3) full version of {\algname} without the averaging of negative augmentation features. As shown in Figure~\ref{fig-ablation}, results across all baselines consistently demonstrate that these ablated variants perform significantly worse than the full version of {\algname}. This demonstrates that batch-wide negative augmentation generation and feature averaging are both essential for {\algname}'s performance gains.

\section{Conclusion}
In this work, we introduce {\algname}, a novel test-time adaptation method that leverages negative data augmentation to mitigate prediction bias caused by image corruptions. Unlike traditional positive augmentation strategies, {\algname} generates negative augmentations by disrupting object semantics through patch shuffling, effectively preserving corruption-specific characteristics while suppressing object-relevant features. By aggregating features from negatively augmented images, our approach offsets the corruption-induced bias in test samples and significantly reduces computational overhead by enabling shared augmentations within each batch. Extensive experiments on standard corruption benchmarks demonstrate that {\algname} consistently outperforms positive data augmentation methods and robustly enhances the performance of various TTA frameworks. Our results highlight the practical effectiveness and efficiency of negative data augmentation for robust vision-language model adaptation.

\section*{Acknowledgement}

This work is supported by National Science Foundation under Award No. IIS-2416070, IIS-2117902. The views and conclusions are those of the authors and should not be interpreted as representing the official policies of the funding agencies or the government.

\bibliography{main}

\appendix

\clearpage
\onecolumn

\section*{Appendix}

\section{Additional Analysis} \label{appendix-a-analysis}

\subsection{Distributional Distance}

To quantify the prediction bias under distribution shift, we compute the distributional distance between the predicted label distribution $ \vp \in \mathbb{R}^C $ and the ground-truth label distribution $ \vq \in \mathbb{R}^C $. We adopt the L1 distance as a simple yet effective metric for measuring this discrepancy:
\begin{align}
\mathrm{L1}(\vp, \vq) = \sum_{c=1}^C |p_c - q_c|.
\end{align}

The ground-truth distribution $ \vq = (q_1, \ldots, q_C) $ is computed by averaging the one-hot encoded labels over the entire test set:
\begin{align}
q_c = \frac{1}{N} \sum_{i=1}^{N} \mathbb{I}(y_i = c), \quad c = 1, \ldots, C,
\end{align}
where $ y_i \in \{1, \ldots, C\} $ is the ground-truth label for the $ i $-th image.

The predicted distribution $ \vp = (p_1, \ldots, p_C) $ is obtained by averaging the softmax outputs over all test samples. Specifically, the logits for each image $ i $ are computed as

\begin{align}
\text{logit}_{i, c} = 100 \cdot \vd_i^\top \vt_c
\end{align}
where $ \vd_i \in \mathbb{R}^D $ is the debiased image feature, and $ \vt_c \in \mathbb{R}^D $ is the text embedding for class $ c $. The predicted probability for class $ c $ is then:
\begin{align}
p_c = \frac{1}{N} \sum_{i=1}^{N} \frac{\exp(\text{logit}_{i, c})}{\sum_{k=1}^{C} \exp(\text{logit}_{i, k})}.
\end{align}

We compute $ \mathrm{L1}(\vp, \vq) $ on ViT-B/32 across all 15 corruption types in CIFAR-10-C. A larger value of this metric indicates a greater deviation of the model’s predicted class distribution from the true label distribution, thus reflecting a higher level of prediction bias under corruption.

As shown in Figure~\ref{fig-l1-all}, the results of the L1 distance reveal the following key conclusions:
(1) When transitioning from raw images (CIFAR-10) to corrupted images (CIFAR-10-C), the distributional distances increase significantly, indicating a substantial rise in prediction bias under corruption.
(2) After applying positive data augmentation (PDA) methods, distributional distances decrease in 4 corruption domains, while they increase to varying degrees in the remaining 11 domains. This suggests that PDA methods fail to consistently reduce prediction bias and may even exacerbate it under certain corruption types.
(3) In contrast, applying {\algname} leads to a consistent and significant reduction in distributional distance across all 15 corruption domains, demonstrating its effectiveness in mitigating prediction bias when recognizing corrupted images.\\

\begin{figure}[!ht]
    \centering
    \includegraphics[width=0.9\linewidth]{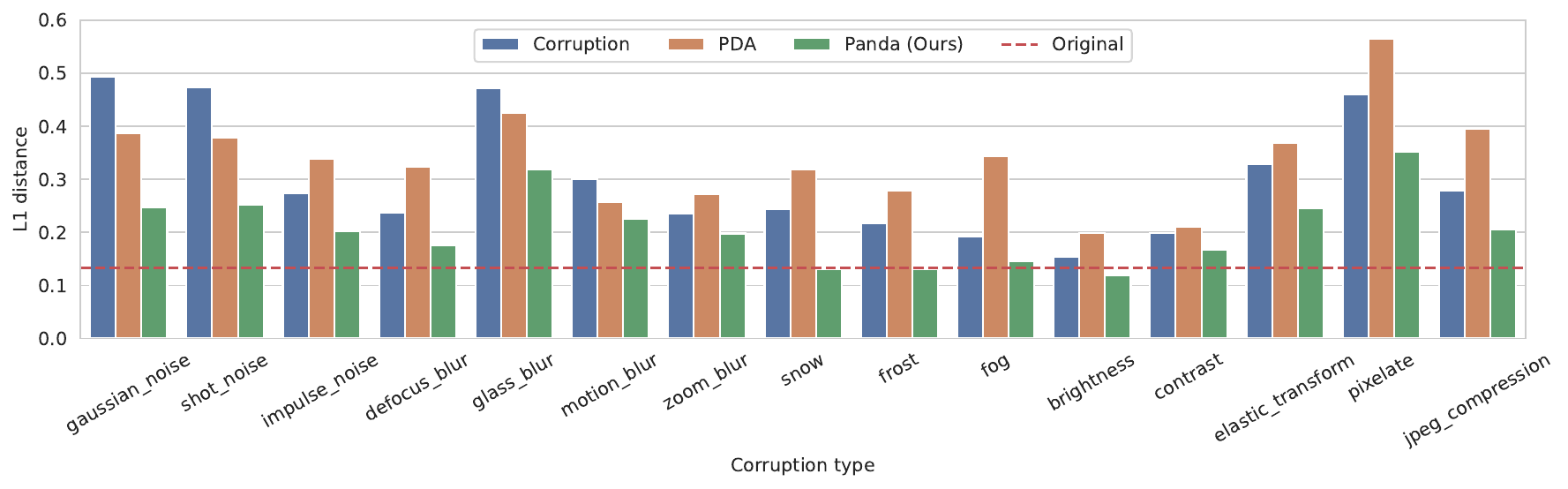}
    \caption{
        L1 distance between soft predictions and ground-truth labels for CLIP, PDA, and {\algname} in 15 types of corruptions under four corruption categories on CIFAR-10-C. Corruptions significantly introduce prediction bias, which positive data augmentation often fail to mitigate. In contrast, our proposed {\algname} effectively reduces such bias.}
    \label{fig-l1-all}
\end{figure}

\newpage

\subsection{Histograms}

To more clearly illustrate the ability of {\algname} to reduce prediction bias on corrupted datasets, we select all corruption domains and visualize the predictions among domains in CIFAR-10-C. In CIFAR-10-C, the number of images per class is uniformly distributed, therefore, if one or more classes are predicted disproportionately more frequently than others to a certain degree, it indicates the presence of prediction bias. As shown in the Figure~\ref{fig-panda-hist-all}, under corruption settings, certain classes are often predicted at a frequency far exceeding the uniform expectation. In contrast, the application of {\algname} significantly mitigates this effect, leading to a more balanced prediction distribution across all classes.

\begin{figure}[h!]
    \centering
    \includegraphics[width=1\linewidth]{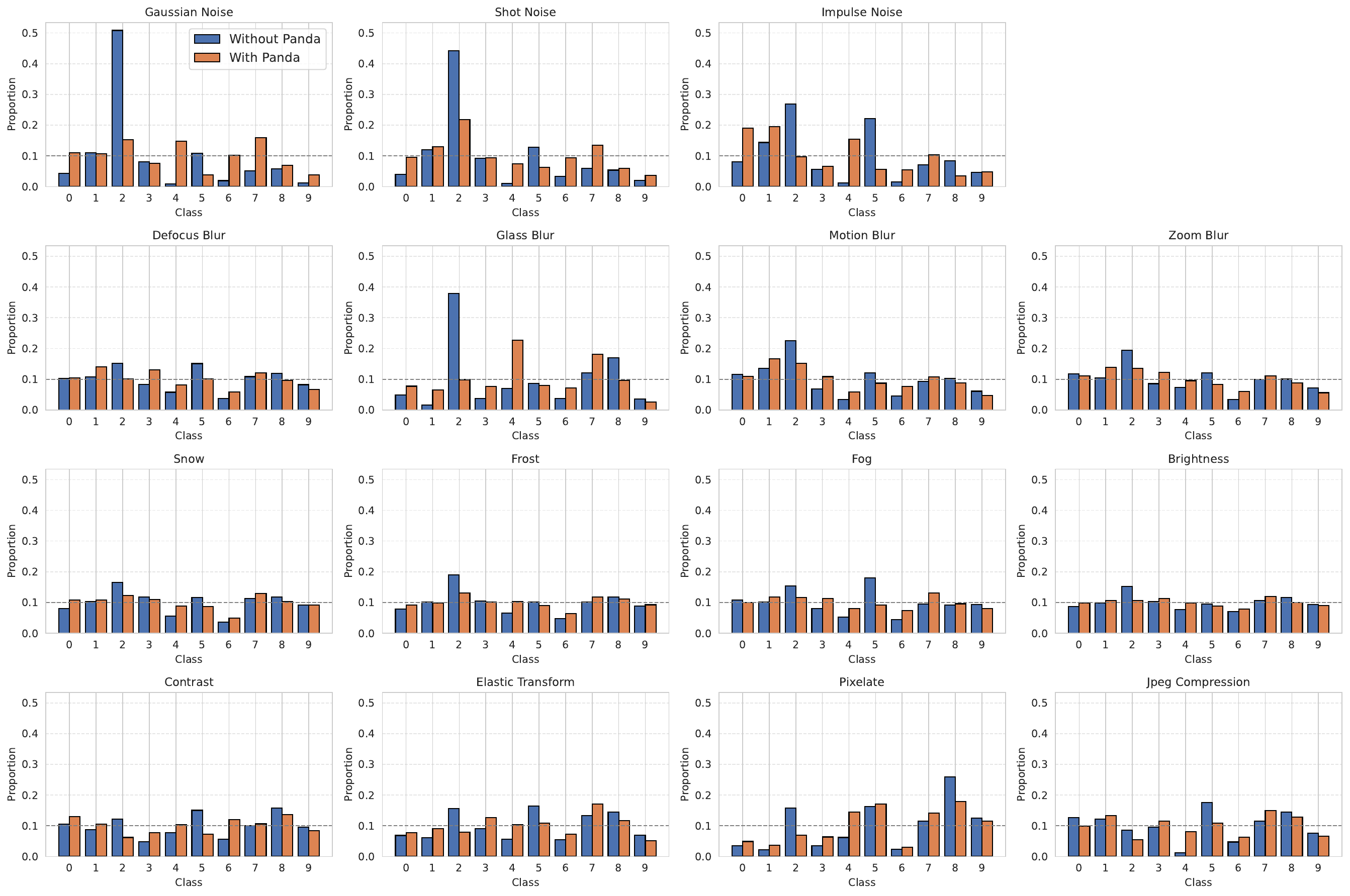}
    \caption{Comparison of predicted class histograms without and with {\algname} across all corruption types in CIFAR-10-C.}
    \label{fig-panda-hist-all}
\end{figure}

\newpage
\section{Proofs}
\label{appendix-c-proof}

In the section, we provide proof of Theorem \ref{thm:offset} below and its generalization to higher dimension in Corollary \ref{crl:offset_hd}.

\thmoffset*

\begin{proof}
    We first reparametrize these components, let
    \begin{align*}
        z_1 = \vcls, \quad z_2 = \frac{1}{s} \cdot \vcorr, \quad z_3 = \frac{1}{s \cdot \sqrt{1 - r^2}} \cdot (n - r \cdot \vcorr), 
    \end{align*}
    with inverse transformation
    \begin{align*}
        \vcls = z_1, \quad \vcorr = s \cdot z_2, \quad n = s\cdot r \cdot z_2 + s \cdot \sqrt{1 - r^2} \cdot z_3. 
    \end{align*}

    \paragraph{W.o offset}
    \begin{align*}
        \Pr(\sign(v) = y) &= \Pr(\sign(\vcls + \vcorr) = \sign(\vcls)) \\
        &= \Pr(\sign(z_1 + s \cdot z_2) = \sign(z_1)) \\
        &= \Pr(z_1 + s \cdot z_2 > 0 \mid z_1 > 0).   \tag{by symmetry} 
    \end{align*}
    Since the joint distribution of \( (z_1, z_2) \) is rotationally symmetric, we convert to polar coordinates:
    \begin{align*}
    z_1 = R \cos\theta, \quad z_2 = R \sin\theta,
    \end{align*}
    where \( R \in [0, \infty) \) and \( \theta \in [-\pi, \pi) \). The inequality conditions become:
    \begin{align*}
        z_1 > 0 &\Rightarrow \cos\theta > 0 \Rightarrow \theta \in \left(-\frac{\pi}{2}, \frac{\pi}{2}\right), \\
        z_1 + s z_2 > 0 &\Rightarrow \cos\theta + s \sin\theta > 0 \Rightarrow \tan\theta > -\frac{1}{s}.
    \end{align*}
    Therefore, the conditional probability is
    \begin{align*}
        \Pr(z_1 + s \cdot z_2 > 0 \mid z_1 > 0) 
        &= \Pr\left(\left. \tan\theta > -\frac{1}{s}\ \right|\  \theta \in \left(-\frac{\pi}{2}, \frac{\pi}{2}\right)\right) \\
        &= \frac{1}{2} \cdot \Pr\left(\left. \tan\theta > -\frac{1}{s}\ \right|\  \theta \in \left(0, \frac{\pi}{2}\right)\right) 
        + \frac{1}{2} \cdot \Pr\left(\left. \tan\theta > -\frac{1}{s}\ \right|\  \theta \in \left(-\frac{\pi}{2}, 0 \right)\right) \\
        &= \frac{1}{2} + \frac{1}{\pi} \cdot \arctan\left(\frac{1}{s}\right). 
    \end{align*}

    \paragraph{W. offset}
    Notice that 
    \begin{align*}
        v - \beta \cdot n &= \vcls + \vcorr - \beta \cdot n \\
        &= z_1 + s \cdot z_2 - \beta \left( s\cdot r \cdot z_2 + s \cdot \sqrt{1 - r^2} \cdot z_3 \right) \\
        &= z_1 + s \cdot (1 - \beta r) \cdot z_2 - s \cdot \beta \cdot \sqrt{1 - r^2} \cdot z_3. 
    \end{align*}
    Let 
    \begin{align*}
        z_4 = \frac{1 - \beta r}{\sqrt{1 - 2\beta r + \beta^2}} \cdot z_2 - \frac{\beta\sqrt{1 - r^2}}{\sqrt{1 - 2\beta r + \beta^2}} \cdot z_3 \sim \gN(0, 1), 
    \end{align*}
    which is independent with $z_1$. We have
    \begin{align*}
        z_1 + s \cdot (1 - \beta r) z_2 - s \cdot \beta \cdot \sqrt{1 - r^2} z_3  = z_1 + s \cdot \sqrt{1 - 2\beta r + \beta^2} \cdot z_4. 
    \end{align*}
    We can treat $s \cdot \sqrt{1 - 2\beta r + \beta^2}$ as a whole. Reusing the results from w.o. offset, we have
    \begin{align*}
        \Pr(\sign(v - \beta \cdot n) = y) &= \frac{1}{2} + \frac{1}{\pi} \arctan \left( \frac{1}{s \cdot \sqrt{1 - 2\beta r + \beta^2}}\right) \\
        &= \frac{1}{2} + \frac{1}{\pi} \arctan \left( \frac{1}{s \cdot \sqrt{1 - r^2 + (\beta - r)^2}}\right). 
    \end{align*}
\end{proof}

\begin{corollary}[Generalization to higher dimension] \label{crl:offset_hd}
    Consider a binary classification problem where the input feature $\vv \in \R^d$ can be decomposed into two independent components $\vv = \vvcls + \vvcorr$, where $\vvcls \sim \gN(\vzero, \mI)$ denotes the class-relevant component and $\vvcorr \sim \gN(\vzero, s^2 \cdot \mI)$ denotes the corruption-related component, with $s > 0$ representing corruption severity. Let the ground-truth label be $y = \sign(\vvcls^\top \vt)$, and the classifier be $\hat{y} = \sign(\vv^\top \vt)$, where $\vt$ is the direction of classifier satisfying $\| \vt \|_2 = 1$. Then the classification accuracy is
    \begin{align}
        \Pr(\sign(\vv) = y) = \frac{1}{2} + \frac{1}{\pi} \cdot \arctan\left(\frac{1}{s}\right).
    \end{align}
    
    Now consider a negatively augmented feature $\vn \sim \gN(\vzero, s^2 \cdot \mI)$ such that the correlation matrix $\rho(\vn, \vvcls) = \vzero$ and $\rho(\vn, \vvcorr) = \mR$. Then offsetting $\vv$ using $\vn$ yields improved accuracy:
    \begin{align}\begin{split}
        &\Pr(\sign(\vv - \beta \cdot \vn) = y) = \frac{1}{2} + \frac{1}{\pi} \cdot \arctan\left(\frac{1}{s \cdot \sqrt{1 - r^2 + (\beta - r)^2}}\right), 
    \end{split}\end{align}
    where $r = \vt^\top \mR \vt$. This is maximized when $\beta = r$. 
\end{corollary}
\begin{proof}
    In the high-dimensional case, the only direction that matters for classification is $\vt$. All components orthogonal to $\vt$ do not affect the classification outcome. Therefore, it suffices to reduce the high-dimensional case to an equivalent one-dimensional case along the $\vt$ direction. Let
    \begin{align*}
        \vcls = \vt^\top \vvcls, \quad \vcorr = \vt^\top \vvcorr, \quad n = \vt^\top \vn
    \end{align*}
    where $\vcls$ and $\vcorr$ are still independent, and we have variances
    \begin{align*}
        \sigma^2(\vcls) = \vt^\top \mI \vt = 1, \quad \sigma^2(\vcorr) = \vt^\top (s^2 \cdot \mI) \vt^\top = s^2, \quad \sigma^2(n) = \vt^\top (s^2 \cdot \mI) \vt^\top = s^2
    \end{align*}
    covariances
    \begin{align*}
        \Cov(\vvcorr, \vn) &= (s \cdot \mI) \mR (s \cdot \mI) = s^2 \cdot \mR
    \end{align*}
    and correlations
    \begin{align*}
        \rho(\vcls, n) &= \frac{\Cov(\vcls, n)}{\sigma(\vcls) \sigma(n)} = 0 \\
        \rho(\vcorr, n) &= \frac{\Cov(\vcorr, n)}{\sigma(\vcorr) \sigma(n)} = \frac{\vt^\top \Cov(\vvcorr, \vn) \vt}{s^2} = \frac{\vt^\top (s^2 \cdot \mR) \vt}{s^2}  = \vt^\top \mR \vt
    \end{align*}
    The proof is then completed by applying Theorem \ref{thm:offset} with $r = \vt^\top \mR \vt$. 
\end{proof}
\begin{remark}
    In the context of CLIP, with two classes' text embeddings $\vt_1, \vt_2$ we can consider $\vt = \frac{\vt_1 - \vt_2}{\| \vt_1 - \vt_2 \|_2}$. 
\end{remark}
\newpage
\section{Additional Experiments} 
\label{appendix-d-exp}

\subsection{Hyperparameter details}

To ensure a fair comparison between baselines and their {\algname}-enhanced counterparts, we apply a strictly controlled and consistent hyperparameter tuning process as follows.
(1) For each baseline, we follow its original hyperparameter recommendations and search within the reported ranges under our setting, selecting a configuration that yields relatively strong performance for a given dataset and architecture.
(2) When configuring baseline+{\algname}, we keep all hyperparameters identical to the corresponding baseline, except for the offset ratio $\beta$, thereby ensuring a fair basis for comparison.
(3) For tuning $\beta$, we perform a grid search over the range $(0.1, 1.0)$ with a step size of 0.1, and report the best-performing result for each baseline+{\algname} combination.\\

The following information is our hyperparameter details for ViT-B/32 in CIFAR-10-C.

\begin{itemize}
\item For all positive data augmentation baselines (TPT~\cite{tpt}, TPS~\cite{tps}, Zero~\cite{zero}), we use AugMix to augment each test image 63 times to obtain a batch of 64 images, which includes the original image. We select 10\% of samples in the batch with lowest entropy to aggregate.
\item In TPT~\cite{tpt}, the number of prompt tokens is 4, the prompt is initialized with ``a photo of a'', and class-specific contexts are disabled. We use the AdamW optimizer and adopt a learning rate of $0.001$, consistent with the setting used for ImageNet in the original papers.
\item In TPS~\cite{tps}, we also use the AdamW optimizer and adopt a learning rate of $0.01$, consistent with the setting used for ImageNet in the original papers.
\item In Zero~\cite{zero}, we follow the same AugMix setup and use 10\% of confident predictions for adaptation.
\item In DMN-ZS~\cite{dmn}, the positive cache is enabled with a shot capacity of 50, an adaptation strength ($\alpha$) of 0.3, and a sharpness ratio ($\beta$) of 5.5.
\item In DeYO~\cite{deyo}, we use a learning rate of $0.001$, an entropy margin multiplier of $0.4$, and a margin of $2.0$. We apply patch-based augmentation with patch length 4, occlusion size 32, and start coordinates at (0, 0). The reweighting coefficients for entropy and PLPD are both set to $1.0$, the entropy reset constant is $0.1$, and we perform 1 update step per batch.
\item In SAR~\cite{sar}, we use a learning rate of $0.05$, an entropy margin multiplier of $0.4$, and a reset constant of $0.2$.
\item In ETA~\cite{eata}, we set the learning rate to $0.0001$, the entropy margin multiplier to $0.5$, and the distance margin to $1.0$. The augmentation type is patch-based (patch length 4, occlusion size 32, start position at (0, 0)). Both reweighting terms are set to $1.0$, the entropy reset constant is $0.1$, and 1 step is applied per batch.
\item In Tent~\cite{tent}, we use a learning rate of $0.0002$.
\item In BAT~\cite{bat}, we use a learning rate of $0.0001$.
\item When combining {\algname} with each baseline, all hyperparameters are kept the same as the original baseline, except for the bias offset $\beta$, which is tuned separately. The values of $\beta$ used are: CLIP: 0.7; DeYO: 0.8; SAR: 0.7; ETA: 0.8; Tent: 0.6; TPS: 0.6; TPT: 0.4; Zero: 0.4; DMN: 0.5; BAT: 0.7.
\end{itemize}

The following information is our hyperparameter details for ViT-B/32 in CIFAR-100-C.

\begin{itemize}
    \item For all positive data augmentation baselines (TPT~\cite{tpt}, TPS~\cite{tps}, Zero~\cite{zero}), we use AugMix to augment each test image 63 times to obtain a batch of 64 images, which includes the original image. We select 10\% of samples in the batch with lowest entropy to aggregate. 
    \item In TPT~\cite{tpt}, the number of prompt tokens is 4, the prompt is initialized with ``a photo of a'', and class-specific contexts are disabled. We use the AdamW optimizer and adopt a learning rate of $0.001$, consistent with the setting used for ImageNet in the original papers. 
    \item In TPS~\cite{tps}, we also use the AdamW optimizer and adopt a learning rate of $0.001$, consistent with the setting used for ImageNet in the original papers.
    \item In Zero~\cite{zero}, we follow the same AugMix setup and use 10\% of confident predictions for adaptation.
    \item In DMN-ZS~\cite{dmn}, the positive cache is enabled with a shot capacity of 50, an adaptation strength ($\alpha$) of 0.3, and a sharpness ratio ($\beta$) of 5.5.
    \item In DeYO~\cite{deyo}, we use a learning rate of $0.001$, an entropy margin multiplier of $0.4$, and a margin of $2.0$. We apply patch-based augmentation with patch length 4, occlusion size 32, and start coordinates at (0, 0). The reweighting coefficients for entropy and PLPD are both set to $1.0$, the entropy reset constant is $0.1$, and we perform 1 update step per batch.
    \item In SAR~\cite{sar}, we use a learning rate of $0.1$, an entropy margin multiplier of $0.4$, and a reset constant of $0.2$.
    \item In ETA~\cite{eta}, we set the learning rate to $0.0005$, the entropy margin multiplier to $0.5$, and the distance margin to $1.0$. The augmentation type is patch-based (patch length 4, occlusion size 32, start position at (0, 0)). Both reweighting terms are set to $1.0$, the entropy reset constant is $0.1$, and 1 step is applied per batch.
    \item In Tent~\cite{tent}, we use a learning rate of $0.0001$.
    \item In BAT~\cite{bat}, we use a learning rate of $0.0002$.
    \item When combining {\algname} with each baseline, all hyperparameters are kept the same as the original baseline, except for the bias offset $\beta$, which is tuned separately. The values of $\beta$ used are: CLIP: 0.4; DeYO: 0.5; SAR: 0.3; ETA: 0.4; Tent: 0.3; TPS: 0.4; TPT: 0.4; Zero: 0.5; DMN: 0.5; BAT: 0.3.
\end{itemize}

The following information is our hyperparameter details for ViT-B/16 in ImageNet-C.

\begin{itemize}
    \item For all positive data augmentation baselines (TPT~\cite{tpt}, TPS~\cite{tps}, Zero~\cite{zero}), we use AugMix to augment each test image 63 times to obtain a batch of 64 images, which includes the original image. We select 10\% of samples in the batch with lowest entropy to aggregate. 
    \item In TPT~\cite{tpt}, the number of prompt tokens is 4, the prompt is initialized with ``a photo of a'', and class-specific contexts are disabled. We use the AdamW optimizer and adopt a learning rate of $0.001$, consistent with the setting used for ImageNet in the original papers. 
    \item In TPS~\cite{tps}, we also use the AdamW optimizer and adopt a learning rate of $0.001$, consistent with the setting used for ImageNet in the original papers.
    \item In Zero~\cite{zero}, we follow the same AugMix setup and use 10\% of confident predictions for adaptation.
    \item In DMN-ZS~\cite{dmn}, the positive cache is enabled with a shot capacity of 50, an adaptation strength ($\alpha$) of 0.3, and a sharpness ratio ($\beta$) of 5.5.
    \item In DeYO~\cite{deyo}, we use a learning rate of $0.001$, an entropy margin multiplier of $0.4$, and a margin of $2.0$. We apply patch-based augmentation with patch length 4, occlusion size 32, and start coordinates at (0, 0). The reweighting coefficients for entropy and PLPD are both set to $1.0$, the entropy reset constant is $0.1$, and we perform 1 update step per batch.
    \item In SAR~\cite{sar}, we use a learning rate of $0.01$, an entropy margin multiplier of $0.4$, and a reset constant of $0.2$.
    \item In ETA~\cite{eta}, we set the learning rate to $0.0001$, the entropy margin multiplier to $0.5$, and the distance margin to $1.0$. The augmentation type is patch-based (patch length 4, occlusion size 32, start position at (0, 0)). Both reweighting terms are set to $1.0$, the entropy reset constant is $0.1$, and 1 step is applied per batch.
    \item In Tent~\cite{tent}, we use a learning rate of $0.0001$.
    \item In BAT~\cite{bat}, we use a learning rate of $0.0002$.
    \item When combining {\algname} with each baseline, all hyperparameters are kept the same as the original baseline, except for the bias offset $\beta$, which is tuned separately. The values of $\beta$ used are: CLIP: 0.4; DeYO: 0.3; SAR: 0.2; ETA: 0.3; Tent: 0.3; TPS: 0.4; TPT: 0.4; Zero: 0.4; DMN: 0.3; BAT: 0.3.
\end{itemize}

\subsection{Experiment with details}

In the experiments of \textit{Combination with TTA Methods (RQ1)} and \textit{Comparison of negative augmentation strategies with other NDA methods}, we perform five independent runs for each setting using random seeds \{0, 1, 2, 3, 4\} to ensure statistical reliability for all 15 types of corruption in CIFAR-10-C, CIFAR-100-C, and ImageNet-C. As the space limitation, the main paper just show the final average prediction accuracy for all 15 types of corruption, but not show the specific statistic result for each corruption domain and its following standard deviation. Therefore, in this section, these detailed experiment result is shown in Table~\ref{tab:acc:main_std}. Besides, we also show the specific statistic result for each corruption domain and its following standard deviation for \textit{Comparison of negative augmentation strategies with other NDA methods} as shown in Table~\ref{table-deyo-std}.

\begin{table*}[htbp]
    \centering
    \resizebox{0.95\linewidth}{!}{
    \setlength{\tabcolsep}{1.0mm}{
        \begin{tabular}{llcccccccccccccccccc}
            \toprule
            & & \multicolumn{16}{c}{ViT-B/32 on CIFAR-10-C} &  & \\
            \cmidrule(lr){3-18} \cmidrule(lr){18-20}
            Method & Venue & \multicolumn{3}{c}{Noise} & \multicolumn{4}{c}{Blur} & \multicolumn{4}{c}{Weather} & \multicolumn{4}{c}{Digital} & \multirow{2.5}{*}{Avg.} & \multirow{2.5}{*}{Incr.} \\
            \cmidrule(lr){3-5} \cmidrule(lr){6-9} \cmidrule(lr){10-13} \cmidrule(lr){14-17}
            & & Gauss. & Shot & Impul. 
            & Defoc. & Glass & Motion & Zoom  
            & Snow & Frost & Fog & Brit. 
            & Contr. & Elastic & Pixel & JPEG \\
            \midrule
            CLIP & ICML'21 & \meansd{35.5}{0.0} & \meansd{39.9}{0.0} & \meansd{43.1}{0.0} & \meansd{69.9}{0.0} & \meansd{41.5}{0.0} & \meansd{64.5}{0.0} & \meansd{70.1}{0.0} & \meansd{70.9}{0.0} & \meansd{72.3}{0.0} & \meansd{66.6}{0.0} & \meansd{81.3}{0.0} & \meansd{64.5}{0.0} & \meansd{59.6}{0.0} & \meansd{48.1}{0.0} & \meansd{56.7}{0.0} & \meansd{59.0}{0.0} \\
            +{\algname} & - & \meansd{43.3}{0.1} & \meansd{46.7}{0.1} & \meansd{43.9}{0.2} & \meansd{71.3}{0.1} & \meansd{45.3}{0.2} & \meansd{66.5}{0.1} & \meansd{72.2}{0.1} & \meansd{72.9}{0.1} & \meansd{74.7}{0.1} & \meansd{68.4}{0.2} & \meansd{82.8}{0.1} & \meansd{66.2}{0.1} & \meansd{61.0}{0.1} & \meansd{51.1}{0.1} & \meansd{58.4}{0.1} & \meansd{61.6}{0.0} & +2.6 \\
            \midrule
            Tent & ICLR'21 & \meansd{27.5}{0.9} & \meansd{34.8}{0.6} & \meansd{47.8}{0.5} & \meansd{75.0}{0.3} & \meansd{36.8}{2.4} & \meansd{64.7}{1.1} & \meansd{73.4}{0.8} & \meansd{76.4}{0.3} & \meansd{78.1}{0.3} & \meansd{74.7}{0.2} & \meansd{87.2}{0.2} & \meansd{76.8}{0.3} & \meansd{67.5}{0.1} & \meansd{59.5}{0.3} & \meansd{62.2}{0.1} & \meansd{62.8}{0.1} \\
            +{\algname} & - & \meansd{54.6}{0.3} & \meansd{58.9}{0.3} & \meansd{53.4}{0.6} & \meansd{77.9}{0.1} & \meansd{60.7}{0.5} & \meansd{75.6}{0.1} & \meansd{78.5}{0.2} & \meansd{79.1}{0.2} & \meansd{80.4}{0.2} & \meansd{78.2}{0.1} & \meansd{87.7}{0.1} & \meansd{78.0}{0.1} & \meansd{70.9}{0.3} & \meansd{67.2}{0.3} & \meansd{64.8}{0.2} & \meansd{71.1}{0.1} & +8.3\\
            \midrule
            ETA & CVPR'22 & \meansd{43.4}{0.3} & \meansd{48.3}{0.6} & \meansd{48.1}{0.2} & \meansd{74.4}{0.2} & \meansd{52.1}{0.2} & \meansd{70.6}{0.3} & \meansd{74.8}{0.1} & \meansd{75.0}{0.1} & \meansd{76.7}{0.3} & \meansd{72.7}{0.2} & \meansd{85.6}{0.1} & \meansd{72.5}{0.3} & \meansd{65.0}{0.2} & \meansd{53.6}{0.3} & \meansd{60.9}{0.2} & \meansd{64.9}{0.1} \\
            +{\algname} & - & \meansd{53.9}{0.2} & \meansd{56.8}{0.2} & \meansd{48.5}{0.2} & \meansd{76.7}{0.2} & \meansd{55.3}{0.2} & \meansd{73.5}{0.2} & \meansd{77.1}{0.2} & \meansd{77.4}{0.2} & \meansd{78.3}{0.1} & \meansd{75.5}{0.2} & \meansd{86.2}{0.2} & \meansd{73.9}{0.3} & \meansd{68.0}{0.1} & \meansd{61.6}{0.2} & \meansd{61.6}{0.2} & \meansd{68.3}{0.1} & +3.4 \\
            \midrule
            SAR & ICLR'23 & \meansd{34.6}{0.7} & \meansd{34.9}{0.5} & \meansd{38.4}{0.7} & \meansd{76.6}{1.0} & \meansd{49.4}{1.3} & \meansd{72.6}{1.3} & \meansd{74.8}{0.6} & \meansd{77.5}{0.7} & \meansd{79.9}{0.2} & \meansd{76.9}{0.4} & \meansd{86.7}{0.1} & \meansd{80.6}{0.4} & \meansd{59.3}{0.7} & \meansd{52.0}{0.9} & \meansd{55.0}{0.7} & \meansd{63.3}{0.3} \\
            +{\algname} & - & \meansd{52.8}{0.7} & \meansd{53.6}{0.9} & \meansd{44.9}{0.5} & \meansd{77.9}{0.4} & \meansd{61.0}{1.7} & \meansd{77.5}{1.0} & \meansd{78.7}{0.3} & \meansd{80.6}{0.3} & \meansd{81.2}{0.3} & \meansd{79.9}{0.3} & \meansd{87.4}{0.2} & \meansd{79.9}{0.2} & \meansd{68.1}{2.0} & \meansd{70.8}{0.4} & \meansd{66.3}{0.2} & \meansd{70.7}{0.2} & +7.4 \\
            \midrule
            DeYO & ICLR'24 & \meansd{47.8}{0.1} & \meansd{51.7}{0.2} & \meansd{48.6}{0.1} & \meansd{74.2}{0.3} & \meansd{52.3}{0.2} & \meansd{70.2}{0.1} & \meansd{74.6}{0.2} & \meansd{75.2}{0.1} & \meansd{76.8}{0.1} & \meansd{72.8}{0.0} & \meansd{85.7}{0.2} & \meansd{72.5}{0.0} & \meansd{64.8}{0.1} & \meansd{54.8}{0.1} & \meansd{60.7}{0.2} & \meansd{65.5}{0.0} \\
            +{\algname} & - & \meansd{51.8}{0.1} & \meansd{55.0}{0.2} & \meansd{48.6}{0.1} & \meansd{75.3}{0.0} & \meansd{54.5}{0.2} & \meansd{72.7}{0.3} & \meansd{76.1}{0.2} & \meansd{75.6}{0.1} & \meansd{76.9}{0.2} & \meansd{73.5}{0.2} & \meansd{85.8}{0.2} & \meansd{73.1}{0.3} & \meansd{67.2}{0.2} & \meansd{59.0}{0.1} & \meansd{62.3}{0.2} & \meansd{67.2}{0.1} & +1.7\\
            \midrule
            TPT & NeurIPS'22 & \meansd{42.7}{0.3} & \meansd{45.4}{0.2} & \meansd{46.7}{0.1} & \meansd{70.8}{0.1} & \meansd{46.1}{0.3} & \meansd{67.3}{0.0} & \meansd{72.1}{0.2} & \meansd{73.5}{0.2} & \meansd{75.8}{0.1} & \meansd{68.2}{0.2} & \meansd{83.9}{0.2} & \meansd{72.0}{0.3} & \meansd{62.3}{0.1} & \meansd{49.8}{0.2} & \meansd{57.0}{0.1} & \meansd{62.2}{0.1} \\
            +{\algname} & - & \meansd{45.9}{0.2} & \meansd{48.9}{0.1} & \meansd{47.9}{0.1} & \meansd{72.7}{0.0} & \meansd{46.6}{0.1} & \meansd{68.7}{0.2} & \meansd{73.8}{0.0} & \meansd{74.8}{0.1} & \meansd{75.8}{0.2} & \meansd{70.0}{0.1} & \meansd{83.9}{0.1} & \meansd{69.6}{0.0} & \meansd{63.5}{0.1} & \meansd{51.9}{0.1} & \meansd{59.1}{0.2} & \meansd{63.5}{0.0} & +1.3 \\
            \midrule
            DMN-ZS & CVPR'24 & \meansd{40.6}{0.5} & \meansd{44.0}{0.3} & \meansd{45.5}{0.2} & \meansd{71.5}{0.1} & \meansd{45.1}{0.2} & \meansd{67.7}{0.4} & \meansd{72.6}{0.2} & \meansd{72.5}{0.3} & \meansd{73.4}{0.2} & \meansd{67.3}{0.1} & \meansd{82.4}{0.3} & \meansd{63.4}{0.2} & \meansd{62.3}{0.1} & \meansd{55.1}{0.3} & \meansd{60.6}{0.1} & \meansd{61.6}{0.0} \\
            +{\algname} & - & \meansd{43.6}{0.4} & \meansd{47.4}{0.3} & \meansd{47.6}{0.3} & \meansd{72.6}{0.1} & \meansd{47.7}{0.3} & \meansd{68.3}{0.3} & \meansd{73.6}{0.2} & \meansd{74.0}{0.1} & \meansd{75.0}{0.2} & \meansd{68.5}{0.4} & \meansd{83.4}{0.1} & \meansd{65.1}{0.4} & \meansd{62.9}{0.2} & \meansd{55.7}{0.1} & \meansd{60.5}{0.2} & \meansd{63.1}{0.1} & +1.5\\
            \midrule
            Zero & NeurIPS'24 & \meansd{45.8}{0.2} & \meansd{49.5}{0.1} & \meansd{48.6}{0.1} & \meansd{67.7}{0.1} & \meansd{48.0}{0.1} & \meansd{68.8}{0.1} & \meansd{71.1}{0.0} & \meansd{71.7}{0.1} & \meansd{74.1}{0.2} & \meansd{68.4}{0.1} & \meansd{82.2}{0.1} & \meansd{80.0}{0.1} & \meansd{63.5}{0.1} & \meansd{51.5}{0.2} & \meansd{57.4}{0.1} & \meansd{63.2}{0.0} \\
            +{\algname} & - & \meansd{50.3}{0.1} & \meansd{53.5}{0.1} & \meansd{49.9}{0.1} & \meansd{69.8}{0.0} & \meansd{49.2}{0.2} & \meansd{70.0}{0.1} & \meansd{73.0}{0.1} & \meansd{73.3}{0.2} & \meansd{73.6}{0.1} & \meansd{70.4}{0.1} & \meansd{82.6}{0.1} & \meansd{81.4}{0.2} & \meansd{64.8}{0.2} & \meansd{53.4}{0.1} & \meansd{58.6}{0.1} & \meansd{64.9}{0.0} & +1.7\\
            \midrule
            TPS & WACV'25 & \meansd{42.8}{0.1} & \meansd{47.7}{0.2} & \meansd{48.6}{0.1} & \meansd{69.8}{0.1} & \meansd{48.7}{0.0} & \meansd{69.5}{0.2} & \meansd{72.8}{0.1} & \meansd{73.4}{0.1} & \meansd{75.6}{0.0} & \meansd{69.6}{0.2} & \meansd{83.0}{0.1} & \meansd{77.6}{0.0} & \meansd{64.2}{0.1} & \meansd{52.5}{0.1} & \meansd{59.4}{0.0} & \meansd{63.7}{0.0} \\
            +{\algname} & - & \meansd{48.9}{0.1} & \meansd{51.9}{0.1} & \meansd{49.3}{0.2} & \meansd{72.2}{0.0} & \meansd{50.8}{0.1} & \meansd{71.1}{0.1} & \meansd{74.7}{0.0} & \meansd{75.0}{0.2} & \meansd{75.1}{0.0} & \meansd{71.5}{0.0} & \meansd{83.6}{0.1} & \meansd{78.9}{0.1} & \meansd{66.0}{0.2} & \meansd{54.8}{0.1} & \meansd{59.5}{0.1} & \meansd{65.6}{0.0} & +1.9 \\
            \midrule
            BAT & ICCV'25 & \meansd{49.9}{0.4} & \meansd{53.5}{0.3} & \meansd{49.1}{0.1} & \meansd{74.5}{0.2} & \meansd{53.1}{0.2} & \meansd{72.3}{0.2} & \meansd{75.2}{0.2} & \meansd{74.5}{0.1} & \meansd{76.2}{0.2} & \meansd{71.5}{0.2} & \meansd{84.9}{0.1} & \meansd{72.0}{0.2} & \meansd{65.2}{0.2} & \meansd{53.5}{0.2} & \meansd{60.7}{0.2} & \meansd{65.7}{0.0} \\
            +{\algname} & - & \meansd{54.6}{0.3} & \meansd{57.6}{0.2} & \meansd{51.5}{0.2} & \meansd{76.7}{0.0} & \meansd{57.1}{0.1} & \meansd{74.0}{0.2} & \meansd{77.4}{0.2} & \meansd{77.1}{0.2} & \meansd{77.9}{0.2} & \meansd{74.5}{0.3} & \meansd{85.9}{0.2} & \meansd{73.6}{0.3} & \meansd{68.1}{0.2} & \meansd{58.6}{0.2} & \meansd{62.3}{0.2} & \meansd{68.5}{0.0} & +2.8 \\
            \midrule 
            & & \multicolumn{16}{c}{ViT-B/32 on CIFAR-100-C} &  & \\
            \cmidrule(lr){3-18} \cmidrule(lr){18-20}
            Method & Venue & \multicolumn{3}{c}{Noise} & \multicolumn{4}{c}{Blur} & \multicolumn{4}{c}{Weather} & \multicolumn{4}{c}{Digital} & \multirow{2.5}{*}{Avg.} & \multirow{2.5}{*}{Incr.} \\
            \cmidrule(lr){3-5} \cmidrule(lr){6-9} \cmidrule(lr){10-13} \cmidrule(lr){14-17}
            & & Gauss. & Shot & Impul. 
            & Defoc. & Glass & Motion & Zoom  
            & Snow & Frost & Fog & Brit. 
            & Contr. & Elastic & Pixel & JPEG \\
            \midrule
            CLIP & ICML'21 & \meansd{16.2}{0.0} & \meansd{17.9}{0.0} & \meansd{17.6}{0.0} & \meansd{39.0}{0.0} & \meansd{17.7}{0.0} & \meansd{38.6}{0.0} & \meansd{43.8}{0.0} & \meansd{42.2}{0.0} & \meansd{43.4}{0.0} & \meansd{39.6}{0.0} & \meansd{50.3}{0.0} & \meansd{29.3}{0.0} & \meansd{28.8}{0.0} & \meansd{22.9}{0.0} & \meansd{29.4}{0.0} & \meansd{31.8}{0.0} \\
            +{\algname} & - & \meansd{18.4}{0.1} & \meansd{19.9}{0.0} & \meansd{19.9}{0.1} & \meansd{40.7}{0.1} & \meansd{19.7}{0.1} & \meansd{40.1}{0.0} & \meansd{44.7}{0.1} & \meansd{43.5}{0.1} & \meansd{44.8}{0.1} & \meansd{40.3}{0.1} & \meansd{51.8}{0.1} & \meansd{31.8}{0.0} & \meansd{29.7}{0.1} & \meansd{24.6}{0.1} & \meansd{31.2}{0.1} & \meansd{33.4}{0.0} & +1.6 \\
            \midrule
            Tent & ICLR'21 & \meansd{15.8}{0.4} & \meansd{18.7}{0.4} & \meansd{17.3}{0.1} & \meansd{46.3}{0.2} & \meansd{11.2}{0.3} & \meansd{44.8}{0.1} & \meansd{50.1}{0.1} & \meansd{47.8}{0.1} & \meansd{46.9}{0.2} & \meansd{45.0}{0.3} & \meansd{58.2}{0.1} & \meansd{39.9}{0.3} & \meansd{32.4}{0.2} & \meansd{27.1}{0.4} & \meansd{34.1}{0.1} & \meansd{35.7}{0.1} \\
            +{\algname} & - & \meansd{22.0}{0.2} & \meansd{24.9}{0.2} & \meansd{23.3}{0.2} & \meansd{47.1}{0.2} & \meansd{20.9}{0.5} & \meansd{45.5}{0.1} & \meansd{50.4}{0.2} & \meansd{48.3}{0.2} & \meansd{48.2}{0.1} & \meansd{46.0}{0.2} & \meansd{59.0}{0.1} & \meansd{41.4}{0.1} & \meansd{33.8}{0.2} & \meansd{30.5}{0.3} & \meansd{35.1}{0.2} & \meansd{38.4}{0.1} & +2.7 \\
            \midrule
            ETA & CVPR'22 & \meansd{26.7}{1.0} & \meansd{27.8}{1.4} & \meansd{23.6}{1.1} & \meansd{50.2}{0.1} & \meansd{7.8}{1.6} & \meansd{49.4}{0.2} & \meansd{52.8}{0.5} & \meansd{52.1}{0.2} & \meansd{52.2}{0.6} & \meansd{50.9}{0.3} & \meansd{63.2}{0.2} & \meansd{50.6}{0.3} & \meansd{32.4}{4.5} & \meansd{35.4}{0.7} & \meansd{37.3}{0.4} & \meansd{40.8}{0.4} \\
            +{\algname} & - & \meansd{29.3}{0.3} & \meansd{31.4}{0.5} & \meansd{27.3}{0.2} & \meansd{50.4}{0.2} & \meansd{26.3}{0.7} & \meansd{49.6}{0.3} & \meansd{52.8}{0.3} & \meansd{51.5}{0.3} & \meansd{51.9}{0.1} & \meansd{51.7}{0.3} & \meansd{63.3}{0.2} & \meansd{50.8}{0.4} & \meansd{36.7}{0.3} & \meansd{37.9}{0.2} & \meansd{37.7}{0.3} & \meansd{43.3}{0.1} & +2.5 \\
            \midrule
            SAR & ICLR'23 & \meansd{22.7}{0.9} & \meansd{24.2}{0.6} & \meansd{17.2}{0.8} & \meansd{49.1}{0.4} & \meansd{11.6}{1.1} & \meansd{45.8}{1.4} & \meansd{51.9}{0.4} & \meansd{51.7}{0.3} & \meansd{52.8}{0.5} & \meansd{51.2}{0.9} & \meansd{62.7}{0.2} & \meansd{52.2}{0.3} & \meansd{30.2}{1.4} & \meansd{31.3}{0.8} & \meansd{32.4}{1.6} & \meansd{39.1}{0.3} \\
            +{\algname} & - & \meansd{26.9}{1.1} & \meansd{24.2}{0.4} & \meansd{25.4}{0.3} & \meansd{49.9}{0.3} & \meansd{18.9}{0.5} & \meansd{48.9}{0.6} & \meansd{52.5}{0.4} & \meansd{51.7}{0.6} & \meansd{52.2}{0.7} & \meansd{52.1}{0.5} & \meansd{63.0}{0.2} & \meansd{53.1}{0.2} & \meansd{36.0}{1.0} & \meansd{35.7}{1.2} & \meansd{35.6}{0.3} & \meansd{41.7}{0.2} & +2.6 \\
            \midrule
            DeYO & ICLR'24 & \meansd{18.0}{1.6} & \meansd{21.7}{1.6} & \meansd{15.3}{1.5} & \meansd{50.4}{0.4} & \meansd{9.4}{1.2} & \meansd{49.9}{0.4} & \meansd{50.7}{1.5} & \meansd{50.9}{0.5} & \meansd{53.1}{0.6} & \meansd{52.2}{0.5} & \meansd{64.2}{0.1} & \meansd{53.7}{0.6} & \meansd{20.3}{1.3} & \meansd{29.2}{1.6} & \meansd{31.0}{1.8} & \meansd{38.0}{0.3} \\
            +{\algname} & - & \meansd{27.3}{1.4} & \meansd{28.2}{1.5} & \meansd{23.9}{0.6} & \meansd{50.8}{0.2} & \meansd{21.5}{1.6} & \meansd{49.9}{0.3} & \meansd{51.9}{0.8} & \meansd{51.5}{0.9} & \meansd{53.3}{0.3} & \meansd{52.3}{0.3} & \meansd{63.7}{0.1} & \meansd{53.3}{0.3} & \meansd{34.1}{1.5} & \meansd{36.2}{1.7} & \meansd{34.3}{0.9} & \meansd{42.1}{0.2} & +4.1 \\
            \midrule
            TPT & NeurIPS'22 & \meansd{15.7}{0.2} & \meansd{17.4}{0.1} & \meansd{17.1}{0.1} & \meansd{39.0}{0.1} & \meansd{18.5}{0.1} & \meansd{38.6}{0.1} & \meansd{44.0}{0.0} & \meansd{43.3}{0.0} & \meansd{44.0}{0.1} & \meansd{40.2}{0.1} & \meansd{50.9}{0.2} & \meansd{31.1}{0.0} & \meansd{30.0}{0.1} & \meansd{22.3}{0.1} & \meansd{29.4}{0.1} & \meansd{32.1}{0.0} \\
            +{\algname} & - & \meansd{18.6}{0.1} & \meansd{20.1}{0.1} & \meansd{20.5}{0.1} & \meansd{41.5}{0.0} & \meansd{20.5}{0.2} & \meansd{40.8}{0.1} & \meansd{45.8}{0.0} & \meansd{44.8}{0.2} & \meansd{46.2}{0.1} & \meansd{41.3}{0.1} & \meansd{52.9}{0.1} & \meansd{34.1}{0.1} & \meansd{30.9}{0.0} & \meansd{25.4}{0.1} & \meansd{31.6}{0.1} & \meansd{34.3}{0.0} & +2.2 \\
            \midrule
            DMN-ZS & CVPR'24 & \meansd{16.5}{0.4} & \meansd{18.2}{0.3} & \meansd{17.9}{0.3} & \meansd{39.5}{0.1} & \meansd{18.0}{0.3} & \meansd{39.0}{0.4} & \meansd{44.4}{0.3} & \meansd{42.7}{0.2} & \meansd{43.8}{0.1} & \meansd{39.9}{0.2} & \meansd{50.7}{0.3} & \meansd{29.9}{0.2} & \meansd{29.1}{0.1} & \meansd{23.3}{0.3} & \meansd{29.9}{0.1} & \meansd{32.2}{0.1} \\
            +{\algname} & - & \meansd{19.0}{0.3} & \meansd{20.2}{0.2} & \meansd{20.6}{0.4} & \meansd{41.5}{0.2} & \meansd{20.6}{0.3} & \meansd{41.2}{0.2} & \meansd{45.8}{0.2} & \meansd{44.5}{0.1} & \meansd{45.7}{0.2} & \meansd{40.7}{0.2} & \meansd{52.6}{0.2} & \meansd{32.8}{0.1} & \meansd{30.7}{0.1} & \meansd{25.8}{0.2} & \meansd{32.1}{0.2} & \meansd{34.2}{0.1} & +2.0 \\
            \midrule
            Zero & NeurIPS'24 & \meansd{9.6}{0.2} & \meansd{10.5}{0.1} & \meansd{12.0}{0.0} & \meansd{33.8}{0.2} & \meansd{18.4}{0.3} & \meansd{34.5}{0.2} & \meansd{38.4}{0.1} & \meansd{41.3}{0.3} & \meansd{42.5}{0.1} & \meansd{37.3}{0.1} & \meansd{47.8}{0.2} & \meansd{44.0}{0.1} & \meansd{30.3}{0.1} & \meansd{17.9}{0.2} & \meansd{24.7}{0.1} & \meansd{29.5}{0.1} \\
            +{\algname} & - & \meansd{15.1}{0.1} & \meansd{16.8}{0.0} & \meansd{18.2}{0.1} & \meansd{38.2}{0.1} & \meansd{20.6}{0.1} & \meansd{37.8}{0.1} & \meansd{42.4}{0.1} & \meansd{42.0}{0.1} & \meansd{43.4}{0.1} & \meansd{40.4}{0.1} & \meansd{50.4}{0.1} & \meansd{46.7}{0.2} & \meansd{32.2}{0.0} & \meansd{25.3}{0.1} & \meansd{29.4}{0.1} & \meansd{33.3}{0.0} & +3.8 \\
            \midrule
            TPS & WACV'25 & \meansd{16.2}{0.1} & \meansd{17.6}{0.1} & \meansd{17.9}{0.1} & \meansd{39.4}{0.1} & \meansd{19.2}{0.3} & \meansd{39.5}{0.2} & \meansd{45.0}{0.1} & \meansd{44.5}{0.0} & \meansd{45.9}{0.1} & \meansd{41.1}{0.1} & \meansd{51.6}{0.0} & \meansd{34.6}{0.2} & \meansd{31.1}{0.1} & \meansd{22.9}{0.1} & \meansd{30.0}{0.2} & \meansd{33.1}{0.1} \\
            +{\algname} & - & \meansd{19.4}{0.1} & \meansd{21.2}{0.1} & \meansd{21.5}{0.0} & \meansd{42.0}{0.1} & \meansd{21.4}{0.2} & \meansd{41.5}{0.1} & \meansd{46.2}{0.1} & \meansd{45.6}{0.0} & \meansd{46.9}{0.0} & \meansd{41.8}{0.1} & \meansd{53.6}{0.0} & \meansd{36.1}{0.1} & \meansd{31.8}{0.1} & \meansd{25.7}{0.0} & \meansd{32.1}{0.2} & \meansd{35.1}{0.0} & +2.0 \\
            \midrule
            BAT & ICCV'25 & \meansd{16.0}{0.2} & \meansd{21.6}{0.3} & \meansd{21.4}{0.1} & \meansd{46.5}{0.1} & \meansd{21.5}{0.3} & \meansd{44.5}{0.3} & \meansd{50.4}{0.3} & \meansd{47.3}{0.2} & \meansd{46.5}{0.3} & \meansd{43.9}{0.2} & \meansd{58.4}{0.2} & \meansd{36.5}{0.5} & \meansd{34.1}{0.1} & \meansd{25.5}{0.2} & \meansd{32.9}{0.2} & \meansd{36.5}{0.0} \\
            +{\algname} & - & \meansd{21.2}{0.1} & \meansd{24.9}{0.1} & \meansd{23.6}{0.2} & \meansd{46.9}{0.3} & \meansd{23.2}{0.3} & \meansd{44.7}{0.2} & \meansd{50.4}{0.1} & \meansd{47.4}{0.2} & \meansd{46.8}{0.2} & \meansd{44.2}{0.3} & \meansd{59.0}{0.4} & \meansd{37.7}{0.2} & \meansd{34.0}{0.2} & \meansd{28.3}{0.2} & \meansd{33.7}{0.2} & \meansd{37.7}{0.0} & +1.2 \\
            \midrule
            & & \multicolumn{16}{c}{ViT-B/16 on ImageNet-C} &  & \\
            \cmidrule(lr){3-18} \cmidrule(lr){18-20}
            Method & Venue & \multicolumn{3}{c}{Noise} & \multicolumn{4}{c}{Blur} & \multicolumn{4}{c}{Weather} & \multicolumn{4}{c}{Digital} & \multirow{2.5}{*}{Avg.} & \multirow{2.5}{*}{Incr.} \\
            \cmidrule(lr){3-5} \cmidrule(lr){6-9} \cmidrule(lr){10-13} \cmidrule(lr){14-17}
            & & Gauss. & Shot & Impul. 
            & Defoc. & Glass & Motion & Zoom  
            & Snow & Frost & Fog & Brit. 
            & Contr. & Elastic & Pixel & JPEG \\
            \midrule
            CLIP & ICML'21 & \meansd{11.2}{0.0} & \meansd{12.6}{0.0} & \meansd{12.0}{0.0} & \meansd{23.3}{0.0} & \meansd{15.2}{0.0} & \meansd{24.4}{0.0} & \meansd{22.6}{0.0} & \meansd{32.4}{0.0} & \meansd{29.8}{0.0} & \meansd{35.9}{0.0} & \meansd{54.0}{0.0} & \meansd{17.3}{0.0} & \meansd{12.7}{0.0} & \meansd{31.0}{0.0} & \meansd{33.4}{0.0} & \meansd{24.5}{0.0} \\
            +{\algname} & - & \meansd{14.6}{0.1} & \meansd{16.7}{0.1} & \meansd{15.6}{0.1} & \meansd{24.1}{0.1} & \meansd{16.0}{0.1} & \meansd{26.2}{0.1} & \meansd{24.3}{0.1} & \meansd{32.9}{0.1} & \meansd{31.1}{0.1} & \meansd{37.4}{0.2} & \meansd{53.3}{0.1} & \meansd{18.9}{0.1} & \meansd{14.8}{0.1} & \meansd{32.9}{0.1} & \meansd{33.6}{0.1} & \meansd{26.2}{0.0} & +1.7 \\
            \midrule
            Tent & ICLR'21 & \meansd{7.1}{0.1} & \meansd{8.3}{0.1} & \meansd{8.4}{0.1} & \meansd{25.4}{0.1} & \meansd{17.9}{0.2} & \meansd{27.3}{0.1} & \meansd{24.9}{0.1} & \meansd{33.2}{0.1} & \meansd{30.2}{0.1} & \meansd{37.7}{0.2} & \meansd{54.6}{0.0} & \meansd{21.4}{0.2} & \meansd{13.9}{0.1} & \meansd{33.9}{0.2} & \meansd{36.0}{0.2} & \meansd{25.3}{0.0} \\
            +{\algname} & - & \meansd{16.9}{0.2} & \meansd{18.8}{0.2} & \meansd{17.3}{0.1} & \meansd{26.1}{0.1} & \meansd{18.8}{0.2} & \meansd{28.5}{0.2} & \meansd{26.3}{0.3} & \meansd{34.4}{0.2} & \meansd{31.5}{0.2} & \meansd{39.1}{0.2} & \meansd{54.4}{0.2} & \meansd{22.4}{0.1} & \meansd{16.3}{0.2} & \meansd{35.2}{0.2} & \meansd{36.6}{0.2} & \meansd{28.2}{0.1} & +2.9 \\
            \midrule
            ETA & CVPR'22 & \meansd{13.6}{0.3} & \meansd{15.4}{0.2} & \meansd{14.0}{0.2} & \meansd{24.9}{0.1} & \meansd{17.7}{0.2} & \meansd{26.7}{0.1} & \meansd{24.5}{0.1} & \meansd{33.3}{0.2} & \meansd{30.8}{0.2} & \meansd{37.5}{0.1} & \meansd{54.7}{0.1} & \meansd{20.6}{0.2} & \meansd{14.6}{0.0} & \meansd{33.6}{0.1} & \meansd{35.7}{0.0} & \meansd{26.5}{0.0} \\
            +{\algname} & - & \meansd{16.4}{0.4} & \meansd{18.3}{0.2} & \meansd{16.8}{0.1} & \meansd{25.7}{0.1} & \meansd{18.3}{0.1} & \meansd{27.9}{0.2} & \meansd{26.1}{0.2} & \meansd{34.1}{0.1} & \meansd{31.7}{0.0} & \meansd{39.0}{0.2} & \meansd{54.4}{0.2} & \meansd{21.7}{0.2} & \meansd{16.3}{0.2} & \meansd{35.0}{0.1} & \meansd{36.1}{0.2} & \meansd{27.9}{0.0} & +1.4 \\
            \midrule
            SAR & ICLR'23 & \meansd{20.2}{0.3} & \meansd{21.7}{0.4} & \meansd{21.0}{0.3} & \meansd{27.3}{0.2} & \meansd{24.8}{0.3} & \meansd{31.9}{0.6} & \meansd{29.2}{0.4} & \meansd{36.9}{0.2} & \meansd{33.1}{0.4} & \meansd{41.9}{0.3} & \meansd{56.2}{0.1} & \meansd{30.5}{0.6} & \meansd{23.0}{0.1} & \meansd{39.5}{0.4} & \meansd{40.0}{0.3} & \meansd{31.8}{0.2} \\
            +{\algname} & - & \meansd{20.9}{0.6} & \meansd{22.8}{0.4} & \meansd{21.9}{0.4} & \meansd{27.3}{0.5} & \meansd{24.8}{0.3} & \meansd{32.5}{0.3} & \meansd{29.8}{0.3} & \meansd{37.2}{0.4} & \meansd{33.8}{0.3} & \meansd{41.9}{0.4} & \meansd{56.0}{0.2} & \meansd{31.1}{0.6} & \meansd{24.7}{0.9} & \meansd{39.7}{0.3} & \meansd{40.4}{0.2} & \meansd{32.4}{0.2} & +0.6 \\
            \midrule
            DeYO & ICLR'24 & \meansd{17.0}{0.3} & \meansd{17.6}{0.4} & \meansd{16.6}{0.1} & \meansd{25.7}{0.1} & \meansd{22.0}{0.2} & \meansd{30.1}{0.2} & \meansd{24.9}{0.1} & \meansd{36.7}{0.1} & \meansd{32.6}{0.1} & \meansd{43.2}{0.2} & \meansd{57.7}{0.3} & \meansd{28.6}{0.2} & \meansd{2.4}{0.3} & \meansd{40.1}{0.2} & \meansd{39.7}{0.3} & \meansd{29.0}{0.2} \\
            +{\algname} & - & \meansd{16.6}{0.5} & \meansd{20.5}{0.3} & \meansd{20.1}{0.1} & \meansd{23.9}{0.1} & \meansd{23.0}{0.3} & \meansd{32.1}{0.2} & \meansd{30.3}{0.0} & \meansd{37.8}{0.2} & \meansd{32.8}{0.1} & \meansd{43.6}{0.2} & \meansd{58.1}{0.3} & \meansd{24.3}{0.1} & \meansd{24.4}{0.2} & \meansd{40.4}{0.3} & \meansd{40.5}{0.3} & \meansd{31.2}{0.2} & +2.2 \\
            \midrule
            TPT & NeurIPS'22 & \meansd{9.9}{0.2} & \meansd{11.2}{0.1} & \meansd{10.7}{0.1} & \meansd{24.1}{0.0} & \meansd{15.7}{0.2} & \meansd{24.9}{0.1} & \meansd{23.8}{0.1} & \meansd{33.3}{0.0} & \meansd{31.9}{0.0} & \meansd{37.0}{0.2} & \meansd{54.9}{0.1} & \meansd{18.3}{0.0} & \meansd{14.0}{0.1} & \meansd{34.0}{0.1} & \meansd{34.2}{0.1} & \meansd{25.2}{0.0} \\
            +{\algname} & - & \meansd{14.8}{0.2} & \meansd{17.0}{0.2} & \meansd{15.7}{0.1} & \meansd{25.0}{0.1} & \meansd{16.8}{0.2} & \meansd{26.7}{0.0} & \meansd{25.3}{0.1} & \meansd{34.2}{0.0} & \meansd{32.5}{0.3} & \meansd{38.5}{0.1} & \meansd{54.5}{0.2} & \meansd{20.0}{0.0} & \meansd{15.8}{0.1} & \meansd{35.2}{0.2} & \meansd{34.6}{0.1} & \meansd{27.1}{0.0} & +1.9 \\
            \midrule
            DMN-ZS & CVPR'24 & \meansd{11.4}{0.4} & \meansd{12.7}{0.3} & \meansd{12.3}{0.2} & \meansd{23.3}{0.0} & \meansd{15.2}{0.1} & \meansd{24.6}{0.1} & \meansd{22.7}{0.0} & \meansd{32.3}{0.1} & \meansd{29.8}{0.0} & \meansd{36.0}{0.2} & \meansd{54.1}{0.3} & \meansd{17.3}{0.2} & \meansd{12.8}{0.1} & \meansd{31.1}{0.1} & \meansd{33.3}{0.1} & \meansd{24.6}{0.1} \\
            +{\algname} & - & \meansd{14.6}{0.2} & \meansd{16.7}{0.2} & \meansd{15.6}{0.1} & \meansd{24.2}{0.0} & \meansd{16.0}{0.1} & \meansd{26.2}{0.0} & \meansd{24.2}{0.0} & \meansd{32.9}{0.1} & \meansd{31.2}{0.1} & \meansd{37.4}{0.2} & \meansd{53.3}{0.0} & \meansd{18.9}{0.1} & \meansd{14.9}{0.1} & \meansd{32.9}{0.2} & \meansd{33.5}{0.0} & \meansd{26.2}{0.0} & +1.6 \\
            \midrule
            Zero & NeurIPS'24 & \meansd{5.4}{0.0} & \meansd{6.4}{0.0} & \meansd{5.5}{0.0} & \meansd{23.0}{0.0} & \meansd{15.1}{0.0} & \meansd{22.2}{0.0} & \meansd{23.1}{0.0} & \meansd{32.6}{0.0} & \meansd{31.5}{0.0} & \meansd{37.1}{0.0} & \meansd{53.1}{0.0} & \meansd{27.5}{0.0} & \meansd{16.4}{0.0} & \meansd{37.4}{0.0} & \meansd{32.7}{0.0} & \meansd{24.6}{0.0} \\
            +{\algname} & - & \meansd{12.3}{0.0} & \meansd{15.4}{0.0} & \meansd{14.0}{0.0} & \meansd{23.8}{0.0} & \meansd{16.1}{0.0} & \meansd{24.1}{0.0} & \meansd{24.4}{0.0} & \meansd{33.2}{0.0} & \meansd{33.0}{0.0} & \meansd{37.6}{0.0} & \meansd{52.4}{0.0} & \meansd{29.5}{0.0} & \meansd{18.5}{0.0} & \meansd{39.1}{0.0} & \meansd{33.6}{0.0} & \meansd{27.1}{0.0} & +2.5 \\
            \midrule
            TPS & WACV'25 & \meansd{9.5}{0.1} & \meansd{10.8}{0.1} & \meansd{10.3}{0.1} & \meansd{23.7}{0.3} & \meansd{15.6}{0.0} & \meansd{24.7}{0.2} & \meansd{23.1}{0.1} & \meansd{33.1}{0.1} & \meansd{31.5}{0.2} & \meansd{37.1}{0.4} & \meansd{54.8}{0.0} & \meansd{20.2}{0.0} & \meansd{13.9}{0.1} & \meansd{34.2}{0.1} & \meansd{34.4}{0.1} & \meansd{25.1}{0.1} \\
            +{\algname} & - & \meansd{15.0}{0.1} & \meansd{17.4}{0.0} & \meansd{15.9}{0.1} & \meansd{24.8}{0.2} & \meansd{16.5}{0.1} & \meansd{26.2}{0.1} & \meansd{25.0}{0.0} & \meansd{34.1}{0.1} & \meansd{32.7}{0.1} & \meansd{38.5}{0.2} & \meansd{54.3}{0.0} & \meansd{20.9}{0.0} & \meansd{15.9}{0.2} & \meansd{36.0}{0.1} & \meansd{34.8}{0.1} & \meansd{27.2}{0.0} & +2.1 \\
            \midrule
            BAT & ICCV'25 & \meansd{15.0}{0.2} & \meansd{16.3}{0.2} & \meansd{15.1}{0.2} & \meansd{23.6}{0.2} & \meansd{16.2}{0.3} & \meansd{26.2}{0.2} & \meansd{23.5}{0.2} & \meansd{31.5}{0.2} & \meansd{27.3}{0.2} & \meansd{36.4}{0.4} & \meansd{52.2}{0.4} & \meansd{19.6}{0.1} & \meansd{14.2}{0.3} & \meansd{32.6}{0.2} & \meansd{34.7}{0.1} & \meansd{25.6}{0.0} \\
            +{\algname} & - & \meansd{17.1}{0.1} & \meansd{18.8}{0.3} & \meansd{17.8}{0.4} & \meansd{24.3}{0.2} & \meansd{17.1}{0.3} & \meansd{27.6}{0.2} & \meansd{24.8}{0.2} & \meansd{32.6}{0.4} & \meansd{29.7}{0.3} & \meansd{38.1}{0.3} & \meansd{53.2}{0.4} & \meansd{21.1}{0.2} & \meansd{16.3}{0.3} & \meansd{34.3}{0.2} & \meansd{35.6}{0.1} & \meansd{27.2}{0.1} & +1.6 \\
            \bottomrule 
        \end{tabular}
    }
    }
    \caption{Accuracy (mean (s.d) \%) on corruption benchmarks including each corruption domains.}
    \label{tab:acc:main_std}
\end{table*}

\begin{table}
    \centering
    \resizebox{1.0\linewidth}{!}{
    \setlength{\tabcolsep}{1.0mm}{
        \begin{tabular}{lccccccccccccccccc}
            \toprule
            & \multicolumn{16}{c}{ViT-B/32 on CIFAR-10-C} \\
            \cmidrule(lr){2-17}
            Method & \multicolumn{3}{c}{Noise} & \multicolumn{4}{c}{Blur} & \multicolumn{4}{c}{Weather} & \multicolumn{4}{c}{Digital} & Avg. \\
            \cmidrule(lr){2-4} \cmidrule(lr){5-8} \cmidrule(lr){9-12} \cmidrule(lr){13-16}
            & Gauss. & Shot & Impul. 
            & Defoc. & Glass & Motion & Zoom  
            & Snow & Frost & Fog & Brit. 
            & Contr. & Elastic & Pixel & JPEG \\
            \midrule
            CLIP & \meansd{35.5}{0.0} & \meansd{39.9}{0.0} & \meansd{43.1}{0.0} & \meansd{69.9}{0.0} & \meansd{41.5}{0.0} & \meansd{64.5}{0.0} & \meansd{70.1}{0.0} & \meansd{70.9}{0.0} & \meansd{72.3}{0.0} & \meansd{66.6}{0.0} & \meansd{81.3}{0.0} & \meansd{64.5}{0.0} & \meansd{59.6}{0.0} & \meansd{48.1}{0.0} & \meansd{56.7}{0.0} & \meansd{59.0}{0.0} \\
            Select \& weight & \meansd{47.8}{0.1} & \meansd{51.7}{0.2} & \meansd{48.6}{0.1} & \meansd{74.2}{0.3} & \meansd{52.3}{0.2} & \meansd{70.2}{0.1} & \meansd{74.6}{0.2} & \meansd{75.2}{0.1} & \meansd{76.8}{0.1} & \meansd{72.8}{0.0} & \meansd{85.7}{0.2} & \meansd{72.5}{0.0} & \meansd{64.8}{0.1} & \meansd{54.8}{0.1} & \meansd{60.7}{0.2} & \meansd{65.5}{0.0} \\
            Offset (Ours) & \meansd{53.9}{0.1} & \meansd{56.8}{0.1} & \meansd{48.5}{0.1} & \meansd{76.7}{0.1} & \meansd{55.3}{0.1} & \meansd{73.5}{0.0} & \meansd{77.1}{0.2} & \meansd{77.4}{0.1} & \meansd{78.3}{0.0} & \meansd{75.5}{0.2} & \meansd{86.2}{0.1} & \meansd{73.9}{0.1} & \meansd{68.0}{0.2} & \meansd{61.6}{0.1} & \meansd{61.6}{0.2} & \meansd{68.3}{0.0} \\
            \midrule
            & \multicolumn{16}{c}{ViT-B/32 on CIFAR-100-C} \\
            \cmidrule(lr){2-17}
            Method & \multicolumn{3}{c}{Noise} & \multicolumn{4}{c}{Blur} & \multicolumn{4}{c}{Weather} & \multicolumn{4}{c}{Digital} & Avg. \\
            \cmidrule(lr){2-4} \cmidrule(lr){5-8} \cmidrule(lr){9-12} \cmidrule(lr){13-16}
            & Gauss. & Shot & Impul. 
            & Defoc. & Glass & Motion & Zoom  
            & Snow & Frost & Fog & Brit. 
            & Contr. & Elastic & Pixel & JPEG \\
            \midrule
            CLIP & \meansd{16.2}{0.0} & \meansd{17.9}{0.0} & \meansd{17.6}{0.0} & \meansd{39.0}{0.0} & \meansd{17.7}{0.0} & \meansd{38.6}{0.0} & \meansd{43.8}{0.0} & \meansd{42.2}{0.0} & \meansd{43.4}{0.0} & \meansd{39.6}{0.0} & \meansd{50.3}{0.0} & \meansd{29.3}{0.0} & \meansd{28.8}{0.0} & \meansd{22.9}{0.0} & \meansd{29.4}{0.0} & \meansd{31.8}{0.0} \\
            Select \& weight & \meansd{18.0}{1.6} & \meansd{21.7}{1.6} & \meansd{15.3}{1.5} & \meansd{50.4}{0.4} & \meansd{9.4}{1.2} & \meansd{49.9}{0.4} & \meansd{50.7}{1.5} & \meansd{50.9}{0.5} & \meansd{53.1}{0.6} & \meansd{52.2}{0.5} & \meansd{64.2}{0.1} & \meansd{53.7}{0.6} & \meansd{20.3}{1.3} & \meansd{29.2}{1.6} & \meansd{31.0}{1.8} & \meansd{38.0}{0.3} \\
            Offset (Ours) & \meansd{29.3}{0.3} & \meansd{31.4}{0.2} & \meansd{27.3}{0.2} & \meansd{50.4}{0.0} & \meansd{26.3}{0.3} & \meansd{49.6}{0.2} & \meansd{52.8}{0.1} & \meansd{51.5}{0.3} & \meansd{51.9}{0.1} & \meansd{51.7}{0.2} & \meansd{63.3}{0.2} & \meansd{50.8}{0.4} & \meansd{36.7}{0.1} & \meansd{37.9}{0.2} & \meansd{37.7}{0.1} & \meansd{43.3}{0.1} \\
            \midrule
            & \multicolumn{16}{c}{ViT-B/16 on ImageNet-C} \\
            \cmidrule(lr){2-17}
            Method & \multicolumn{3}{c}{Noise} & \multicolumn{4}{c}{Blur} & \multicolumn{4}{c}{Weather} & \multicolumn{4}{c}{Digital} & Avg. \\
            \cmidrule(lr){2-4} \cmidrule(lr){5-8} \cmidrule(lr){9-12} \cmidrule(lr){13-16}
            & Gauss. & Shot & Impul. 
            & Defoc. & Glass & Motion & Zoom  
            & Snow & Frost & Fog & Brit. 
            & Contr. & Elastic & Pixel & JPEG \\
            \midrule
            CLIP & \meansd{11.2}{0.0} & \meansd{12.6}{0.0} & \meansd{12.0}{0.0} & \meansd{23.3}{0.0} & \meansd{15.2}{0.0} & \meansd{24.4}{0.0} & \meansd{22.6}{0.0} & \meansd{32.4}{0.0} & \meansd{29.8}{0.0} & \meansd{35.9}{0.0} & \meansd{54.0}{0.0} & \meansd{17.3}{0.0} & \meansd{12.7}{0.0} & \meansd{31.0}{0.0} & \meansd{33.4}{0.0} & \meansd{24.5}{0.0} \\
            Select \& weight & \meansd{17.0}{0.3} & \meansd{17.6}{0.4} & \meansd{16.6}{0.1} & \meansd{25.7}{0.1} & \meansd{22.0}{0.2} & \meansd{30.1}{0.2} & \meansd{24.9}{0.1} & \meansd{36.7}{0.1} & \meansd{32.6}{0.1} & \meansd{43.2}{0.2} & \meansd{57.7}{0.3} & \meansd{28.6}{0.2} & \meansd{2.4}{0.3} & \meansd{40.1}{0.2} & \meansd{39.7}{0.3} & \meansd{29.0}{0.2} \\
            Offset (Ours) & \meansd{18.1}{0.2} & \meansd{19.6}{0.2} & \meansd{18.5}{0.0} & \meansd{27.0}{0.1} & \meansd{20.0}{0.2} & \meansd{29.2}{0.2} & \meansd{27.8}{0.1} & \meansd{35.4}{0.0} & \meansd{33.4}{0.2} & \meansd{40.3}{0.2} & \meansd{56.1}{0.2} & \meansd{23.0}{0.3} & \meansd{18.0}{0.1} & \meansd{36.3}{0.1} & \meansd{37.6}{0.2} & \meansd{29.4}{0.0} \\
            \bottomrule 
        \end{tabular}
        
    }
    }
    \caption{Comparison of accuracy (mean (s.d.) \%) between negative augmentation strategies in DeYO and {\algname} including each corruption domain. Selecting and weighting testing samples using NDA is from DeYO, and offsetting bias features is from {\algname}.}
    \label{table-deyo-std}
\end{table}

\subsection{Hyperparameter sensitivity experiment}

Due to space limitations, the hyperparameter sensitivity experiment in term of learning rate mentioned in the main text is presented here in details. These hyperparameter sensitivity experiment, focusing on the stability of {\algname} under different learning rates when combined with other methods. Specifically, we evaluate {\algname} in conjunction with two representative baselines, Tent and SAR. As shown in Figure~\ref{fig-tentpanda-lr} and Figure~\ref{fig-sarpanda-lr}. Within a relatively wide range of learning rates, both Tent+{\algname} and Tent+{\algname} can maintain consistently high performance.

\begin{figure}[ht]
    \centering
    \includegraphics[width=0.5\linewidth]{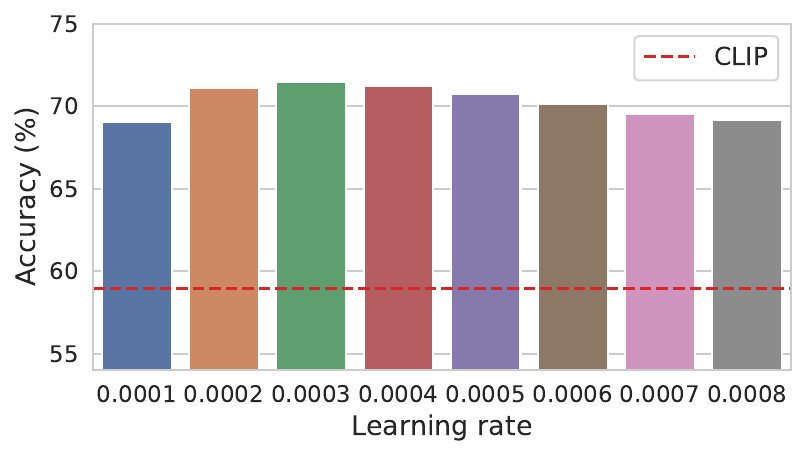}
    \caption{Accuracy of Tent combined with {\algname} under different learning rates on CIFAR-10-C with ViT-B/32.}
    \label{fig-tentpanda-lr}
\end{figure}

\begin{figure}[ht]
    \centering
    \includegraphics[width=0.5\linewidth]{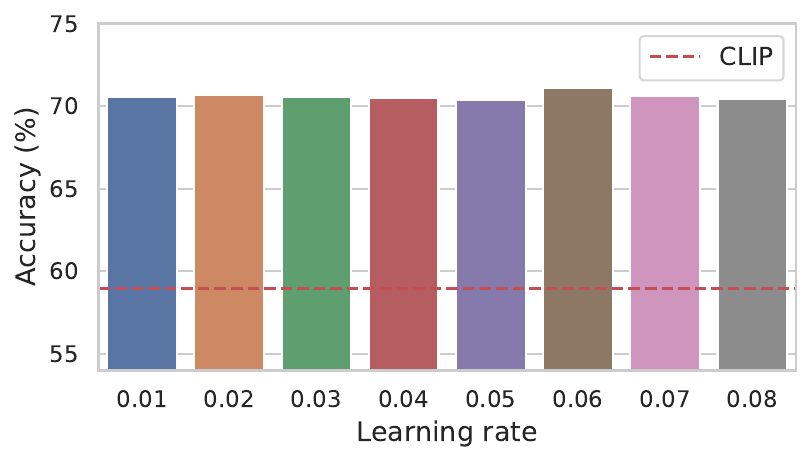}
    \caption{Accuracy of SAR combined with {\algname} under different learning rates on CIFAR-10-C with ViT-B/32.}
    \label{fig-sarpanda-lr}
\end{figure}

\subsection{Infrastructure setting for experiments}

All experiments are conducted on a single NVIDIA Tesla V100 GPU with 32GB memory, except for large-batch-size settings, which are run on a single NVIDIA Tesla A100 GPU with 80GB memory. Experiments are implemented using PyTorch 2.0.1 with CUDA 11.7 and Python 3.9, and executed within the PyCharm 2025.1.3.1 environment.

\end{document}